\documentclass[table]{gtech}
\PassOptionsToPackage{table,usenames,dvipsnames}{xcolor}
\usepackage{xcolor}
\usepackage{amssymb}
\usepackage{multirow}
\usepackage{comment}
\usepackage{bigdelim}
\usepackage{longtable}
\usepackage{tabularray}
\usepackage{wrapfig}
\usepackage[most]{tcolorbox}
\usepackage{url}
\usepackage{float}
\usepackage{datatool}
\usepackage{enumitem}
\usepackage{subcaption} 
\usepackage{amsthm}
\newtheorem{theorem}{Theorem}
\usepackage{pifont}
\usepackage[linesnumbered,ruled,vlined]{algorithm2e}
\usepackage[justification=centering]{caption}

\RequirePackage{tgpagella} 
\RequirePackage{mathpazo}  
\RequirePackage{inconsolata}
\usepackage{makecell}
\usepackage{tcolorbox}

\renewcommand{\title}[1]{\newcommand{\titlelist}{{\huge\fontfamily{optimistic}\selectfont #1}}}

\newcommand{\model}{\texttt{Ring-1T}}
\definecolor{prompt}{HTML}{5f84e4}
\definecolor{img}{HTML}{820100}
\newcommand{\best}[1]{{\textbf{#1}}}
\newcommand{\second}[1]{{\underline{#1}}}

\usepackage{arydshln}
\definecolor{CQColor}{rgb}{0.0,0.0,1.0} 
\definecolor{TABLE_LINE}{rgb}{0.88,0.96,1.0}

\usepackage{colortbl}
\usepackage{amssymb}
\usepackage{pifont}
\usepackage{booktabs,multirow}
\usepackage{makecell}
\usepackage{tabulary}
\usepackage{fontawesome5}
\usepackage{bbding}
\usepackage{multicol}

\newlength\savewidth

\title{\textcolor[HTML]{0369ff}{Every Step Evolves}: Scaling Reinforcement Learning for Trillion-Scale Thinking Model}

\author[*]{Ling Team, Inclusion AI}

\contribution[*]{See Contributions section (Sec.~\ref{sec:contri}) for full author list.}

\abstract{\fontsize{11pt}{12pt} 
We present \model{}, the first open-source, state-of-the-art thinking model with a trillion-scale parameter. It features 1 trillion total parameters and activates approximately 50 billion per token. Training such models at a trillion-parameter scale introduces unprecedented challenges, including train-inference misalignment, inefficiencies in rollout processing, and bottlenecks in the RL system. To address these, we pioneer three interconnected innovations: (1) IcePop stabilizes RL training via token-level discrepancy masking and clipping, resolving instability from training-inference mismatches; (2) C3PO++ improves resource utilization for long rollouts under a token budget by dynamically partitioning them, thereby obtaining high time efficiency; and (3) ASystem, a high-performance RL framework designed to overcome the systemic bottlenecks that impede trillion-parameter model training. \model{} delivers breakthrough results across critical benchmarks: 93.4 on AIME-2025, 86.72 on HMMT-2025, 2088 on CodeForces, and 55.94 on ARC-AGI-1. Notably, it attains a silver medal-level result on the IMO-2025, underscoring its exceptional reasoning capabilities. By releasing the complete 1T parameter MoE model to the community, we provide the research community with direct access to cutting-edge reasoning capabilities. This contribution marks a significant milestone in democratizing large-scale reasoning intelligence and establishes a new baseline for open-source model performance. 
}
\date{Oct 22, 2025\vspace{-1mm}}
\gtechdata[Code]{\url{https://github.com/inclusionAI/Ring-V2}\vspace{-1mm}}
\gtechdata[Model]{\url{https://huggingface.co/inclusionAI/Ring-1T}}

\begin{document}
\maketitle

\begin{figure*}[h]
\centering
\includegraphics[width=0.99\textwidth]{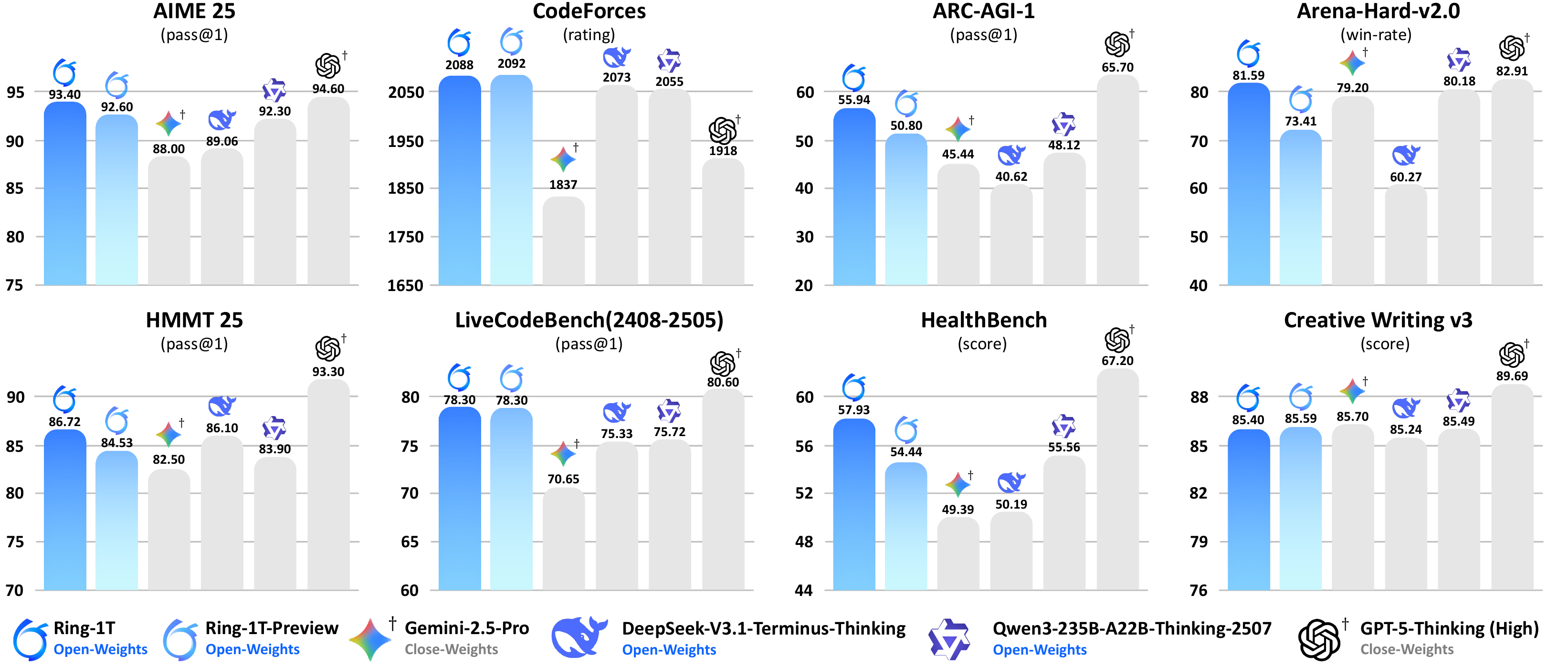}
\caption{Performance comparison of \model{} and existing open-weights and close-weights$^{\dagger}$ models across benchmarks.}
\label{fig:ring-lite-performance}
\end{figure*}

\section{Introduction}
\label{sec:intro}

Artificial intelligence is undergoing a pivotal transition: Large Language Models (LLMs) are advancing beyond static corpora of human knowledge, becoming dynamic processors that transform information into actionable insights and understanding~\citep{kimiteam2025kimik2openagentic,deepseekai2025deepseekr1}. This progression towards more general intelligence is empirically validated by their core capability – complex, adaptive problem-solving. 
Recent breakthroughs in solving high-difficulty human competition problems provide concrete evidence of significantly advanced reasoning abilities in large language models. For instance, models~\citep{gpt5, qwen3max} have achieved 100\% accuracy on the AIME-2025~\citep{aime} and HMMT-2025~\footnote{https://www.hmmt.org/www/archive/problems}, and reached medal-level performance at the International Mathematical Olympiad (IMO)~\citep{gpt5}—a hallmark of sophisticated human intellect. This evolution beyond static knowledge repositories is driven by training on trillions of tokens across diverse domains, coupled with reinforcement learning-optimized reasoning techniques~\citep{openai2024openaio1card, deepseekai2025deepseekr1} that enable models to dynamically scale their capabilities with thinking effort, pointing toward higher levels of general intelligence.

While related work~\citep{deepseekai2025deepseekr1,glm46} has made valuable contributions to the open-source community, the frontier of trillion-parameter thinking models remains uncharted territory. Scaling to this level introduces formidable challenges, such as severe training instability and prohibitive computational costs. In this work, we introduce \model{}, a novel Mixture-of-Experts (MoE) thinking model scaled to unprecedented size—and demonstrate breakthrough methodologies for efficient trillion-parameter training. By solving fundamental stability and efficiency challenges at this scale, we enable robust large-scale reasoning training while providing extensive implementation insights.

Our \model{}, the first open-source reasoning model with one trillion total parameters, is built upon the Ling 2.0~\citep{lingv2} architecture and trained from the Ling-1T-base. 
With approximately 50 billion activated parameters per token, \model{} achieves state-of-the-art performance across multiple challenging benchmarks—despite relying solely on natural language reasoning capabilities. It significantly outperforms existing open-source models, achieving scores of 93.4 on AIME-2025, 86.72 on HMMT-2025, 2088 on CodeForces, and 55.94 on ARC-AGI-v1. Remarkably, in the IMO-2025 evaluation within AWorld~\footnote{https://github.com/inclusionAI/AWorld}, \model{} achieved a silver medal-level result by correctly solving four problems and partially proving Problem 2, all within a single submission, and without relying on code generation or external symbolic solvers.
Realizing this breakthrough required addressing fundamental challenges in trillion-scale RL training. We pioneered three interconnected innovations:
\begin{itemize}
    \item \textbf{IcePop} eliminates catastrophic training-inference misalignment in RL training by clipping excessive-discrepancy tokens. This selective correction pops out unstable contributions while preserving efficient updates, thereby stabilizing training without slowing inference.
    \item \textbf{C3PO++} introduces a budget-controlled rollout scheduling mechanism that eliminates rollout-stage bottlenecks. Thus, it avoids inefficient single-pass processing of oversized sequences, reducing computational overhead while enabling their efficient reuse through batched continuation.
    \item \textbf{ASystem} is a high-performance reinforcement learning (RL) framework designed for large-scale asynchronous training. It adopts a SingleController + SPMD (Single Program, Multiple Data) to enable fully asynchronous operations, multi-phase masking acceleration, and efficient data packing/sharding. 
\end{itemize}

The structure of this paper is organized as follows: Section~\ref{sec:method} describes our comprehensive training methodology, which includes Long Chain-of-Thought Supervised Fine-Tuning (Long-CoT SFT) and large-scale reinforcement learning (RL), including our key algorithmic contributions, IcePop and C3PO++, as well as the underlying training framework, Asystem. Finally, Section~\ref{sec:eval} presents a thorough evaluation of our model's performance against leading open-weights and closed-weights models on established benchmarks.
\section{Approach}
\label{sec:method}
The \model{} model was developed via a multi-stage pipeline, commencing with long-CoT SFT and advancing through reasoning and general RL phases. 
While the complete pipeline is integral to the model's performance, the primary focus of this report is on RL components. We first summarize the long-CoT SFT process that primes the model for RL in Section~\ref{sec:sft}. 
The core of our discussion then turns to the novel RL algorithm, which is elaborated upon in Section~\ref{sec:rl}.
Additionally, we introduce ASystem in Section~\ref{subsec:asys}, the distributed system that enabled the large-scale RL training necessary for this project.

\subsection{Training Pipeline}
\label{sec:pipeline}

\begin{figure}[!hbt]
    \centering
    \includegraphics[width=0.86\linewidth]{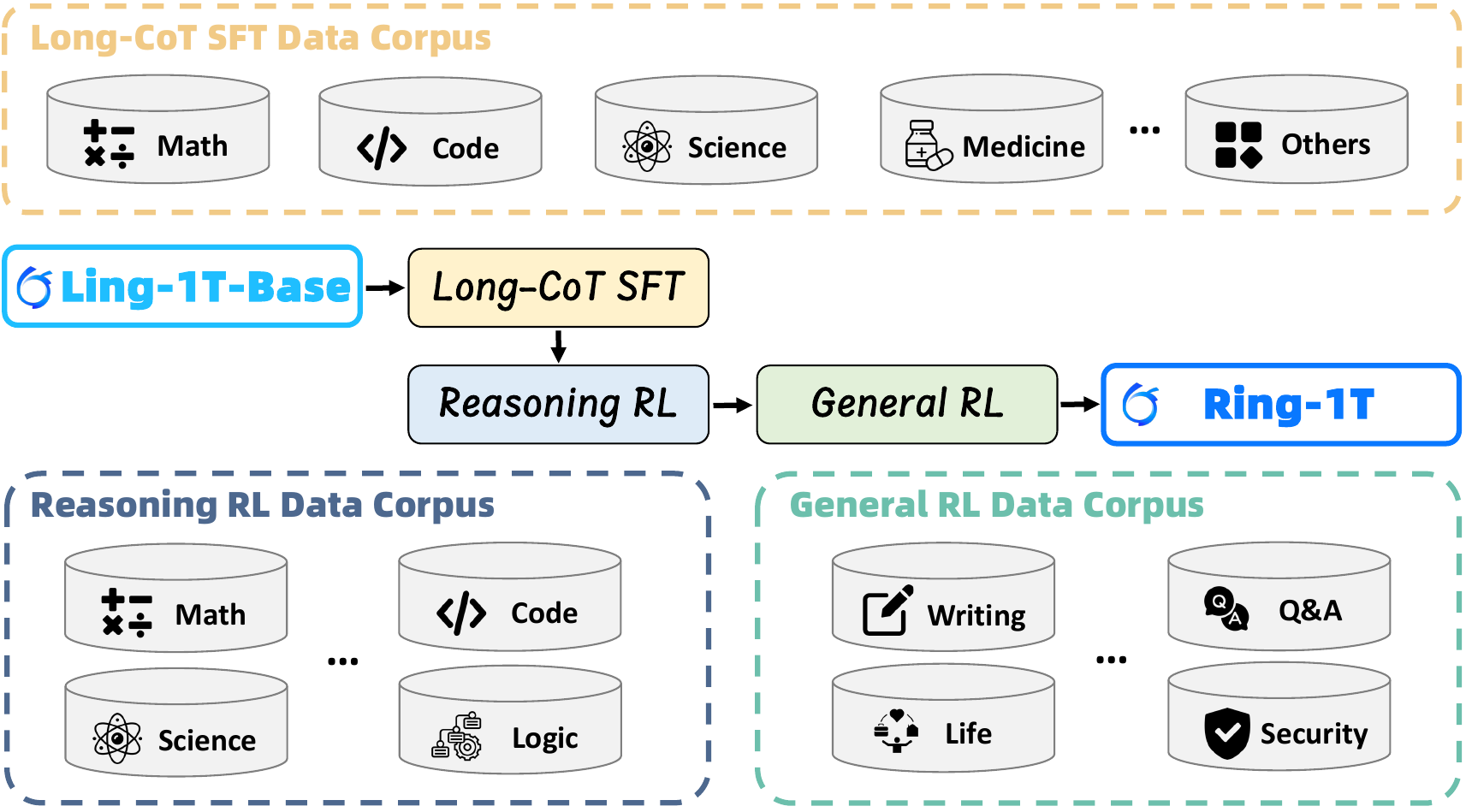}
    \caption{The training pipeline of \model.}
    \label{fig:training_framework}
\end{figure}

We use Ling-1T-base model~\citep{lingv2}, a novel Mixture-of-Experts model with a total of \textbf{1 Trillion} parameters and an activation of \textbf{50 Billion} parameters, as our base model. The training of \model{} consists of three stages, comprising long-CoT SFT, reasoning-oriented RL, and general-oriented RL, to cultivate a powerful thinking model.
\begin{itemize}
    \item \textbf{Long-CoT SFT:} We collect and synthesize a large amount of multi-domain reasoning trajectory data covering mathematics, code, science, etc. Through large-scale supervised fine-tuning, the model learns general reasoning patterns and domain-specific reasoning skills, establishing a solid foundation for large-scale RL training.
    \item \textbf{Reasoning RL:} We construct a comprehensive, challenging, and high-quality RL dataset encompassing math, code, science, and logic tasks with verifiable outcomes. 
    The model's comprehensive reasoning performance is enhanced via \textit{RLVR (Reinforcement Learning from Verifiable Rewards)}. This process involves sampling extensive reasoning trajectories and refining the policy using verifiable rewards, which are provided by carefully designed multi-domain verifiers.
    \item \textbf{General RL:} Following large-scale reinforcement learning on verifiable tasks, we conduct a second RL stage focused on general tasks. This phase employs \textit{RLHF (Reinforcement Learning from Human Feedback)} to recalibrate the model's capability distribution, preserving its core reasoning strength while enhancing human alignment, instruction following, creative writing, safety, and overall usability.
    
\end{itemize}

\subsection{Long-CoT SFT}
\label{sec:sft}

In this stage, we aim to endow the base model with fundamental long-chain reasoning abilities through Long Chain-of-Thought Supervised Fine-Tuning (Long-CoT SFT). This process serves as the foundation for subsequent reinforcement learning, equipping it with the capability to sustain coherent, multi-step thinking processes for complex problems.

\paragraph{Data Collection} A comprehensive, high-quality dataset featuring Long-CoT reasoning patterns was constructed to effectively activate the base model’s reasoning capability. The query pool originates from a tripartite sourcing strategy: open-source repositories, expert manual generation, and LLM-based synthesis. 
Following data collection, rigorous data-cleansing protocols were applied. 
To ensure the resulting dataset's quality, we designed a rigorous data processing pipeline comprising four sequential steps: 1) Deduplication, where we employed exact matching to remove repetitive samples; 2) Harmful Content Filtering, where data samples containing toxic or harmful information were identified and purged; 3) Data Decontamination, where we utilized both hashing and exact string matching techniques to detect and eliminate any samples that overlap with existing benchmarks; and 4) Low-Quality Sample Filtering, which removeed various noise sources including invisible control codes and extraneous Unicode characters. The final data is predominantly composed of four domains: Mathematics (46\%), STEM (26\%), Code (20\%), and Others (8\%).

\paragraph{Training} We conduct Long-CoT SFT on our Ling-1T-base model~\citep{lingv2} to obtain a model with preliminary thinking capabilities. The training data are packed into 64k-length sequences. For this stage, the model was trained for 3 epochs with a learning rate of $2 \times 10^{-4}$. We employed a cosine decay scheduler with 30 warmup steps and applied a weight decay of 0.1 throughout the process.

\subsection{Large-Scale Reinforcement Learning}
\label{sec:rl}
Reinforcement learning (RL) is the critical next step for translating knowledge from pre-training and SFT into advanced thinking. To enable RL at the unprecedented scale of \textbf{1 Trillion} parameters, we developed a novel RL algorithm and a specialized infrastructure, Asystem. This integrated solution overcomes fundamental challenges in training efficiency and stability, allowing us to systematically explore the model's complex problem-solving capabilities.

\subsubsection{RL Data}
\label{subsec:rl-data}

High-quality and diverse data are critical for effective reinforcement learning. To this end, we introduce a carefully curated, multi-domain dataset spanning five core areas: math, code, science, logic, and general domains:
\begin{itemize}
    \item \textbf{Math:} We extend the dataset from~\cite{team2025ring} with mathematically rigorous problems from authoritative sources. Our curation ensures completeness, high complexity, and verifiable solutions, yielding a high-quality corpus for large-scale reinforcement learning.
    
    \item \textbf{Code:} Aside from the dataset employed in~\cite{team2025ring}, we develop a multi-phase workflow for synthesizing, validating, quality-scoring, and selecting additional test cases. This process ensures that each problem is equipped with a sufficient number of high-quality test cases. The final dataset contains programming problems with verified correct solutions and carefully tested cases.
     
    \item \textbf{Science:} We developed a crowdsourced science dataset of high-difficulty problems spanning physics, chemistry, and biology. To ensure complexity for reinforcement learning, all multiple-choice questions were reformatted into an open-ended format. 
    For organic chemistry, we established a dedicated image-semantization pipeline that converts visual information such as molecular structures into structured textual descriptions. 
    Finally, we applied a Pass-rate filtering strategy to select only the highest-quality items.
    
    \item \textbf{Logic:} Our logic reasoning dataset spans five domains: visual pattern induction ~\citep{Chollet2025}, grid puzzles (Sudoku), pathfinding (mazes), arithmetic reasoning (24 Game), and propositional logic (Knights and Knaves). We synthesized problems by integrating public resources such as \cite{Hodel2024}, \cite{Li2025InternBootcamp}, and \cite{Liu2025SynLogic} into an in-house game generator, enabling scalable and controlled creation. A quality control process ensures each task is solvable and non-trivial during both generation and post-processing. The final curated collection balanced across domains and complexity levels for reinforcement learning.

    \item \textbf{General Data:} We constructed a comprehensive dataset for general reasoning by aggregating problems from two primary sources: public repositories and real-world user interactions. From public sources, we incorporated established general datasets including Magpie~\citep{xu2024magpiealignmentdatasynthesis}, WMT~\citep{feng2025mtr1zero}, RLVR-IFEval~\footnote{https://huggingface.co/datasets/allenai/RLVR-IFeval}, and AutoIF~\footnote{https://huggingface.co/datasets/Post-training-Data-Flywheel/AutoIF-instruct-61k}. To enhance practical alignment, we further integrated real-world user preference data such as arena-human-preference-100k and arena-human-preference-140k~\footnote{https://huggingface.co/datasets/lmarena-ai}. Additionally, we augmented this collection with problems sourced from social media platforms such as Zhihu and StackOverflow. 
\end{itemize}

Finally, we employ a multi-stage curation pipeline involving parsing, reformulation, and deduplication, with quality assured by a dual scoring system of LLMs and rule-based metrics. Furthermore, fine-grained metadata annotations on each sample enable dynamic sampling and cross-domain blending, a strategy that significantly improves training efficiency and model performance on complex tasks.

\begin{figure}
    \centering
    \includegraphics[width=0.96\linewidth]{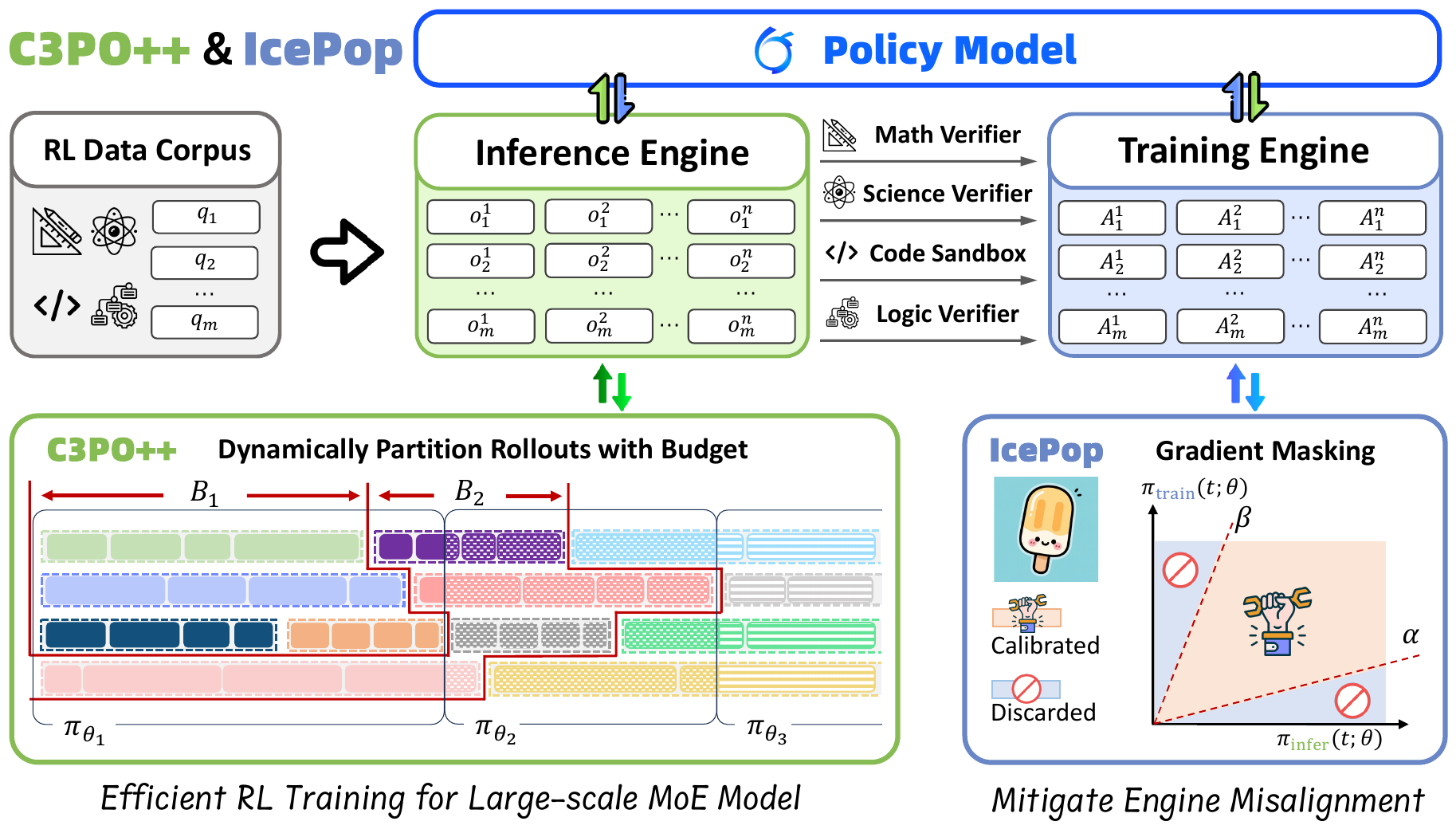}
    \caption{We integrate C3PO++ and IcePop into \model, which enhances both training efficiency and effectiveness of RL.}
    \label{fig:rl_overview}
\end{figure}

\subsubsection{IcePop: Discard All Noisy Gradient Updates}
Contemporary reinforcement learning training frameworks typically utilize distinct engines for model training and inference processes. Throughout our experiments, we have observed that this separation may lead to discrepancies in probability calculations, potentially introducing instability to RL training. This problem is particularly pronounced in the training of MoE models with RL due to the inherent usage of the dynamic routing mechanism. Additionally, in long CoT settings, these discrepancies can gradually accumulate across iterations and become further amplified.

\begin{theorem}{(Compounding Probability Discrepancy)}\label{theo:prob_dis}
Let $\pi_{\mathrm{infer}}(\cdot;\theta)$ and $\pi_{\mathrm{train}}(\cdot;\theta)$ be the policy model loaded by inference and training engines, and $\delta_t \;=\; D_{\mathrm{KL}}\!\big(\pi_{\mathrm{infer}}(\cdot;\theta_t)\,\|\,\pi_{\mathrm{train}}(\cdot;\theta_t)\big)$ be probability discrepancy at step $t$. 
Under certain conditions and a step size $\mu>0$,
there exist a constant $\eta>0$ such that $\delta_{t+1} \;\ge\; \big(1 + \tfrac{\eta}{2}\,\mu\big)\,\delta_t.$ 
\end{theorem}

To address this compounding mismatch issue in MoE RL, we propose \textbf{IcePop}, a variant of GRPO that suppresses unstable training updates through double-sided masking calibration. IcePop only calibrates gradients within the acceptable region and discards all noisy gradient updates beyond that boundary, effectively aligning $\pi_{\textcolor{blue}{\text{train}}}$ with $\pi_{\textcolor{teal}{\text{infer}}}$. This is achieved through two key techniques:
\begin{itemize}
    \item \textbf{Double-sided calibration}: We calibrate token-level gradients within a region defined by lower and upper limits, well preserving the alignments between training and inference probabilities.
    \item \textbf{Masking}: We exclude tokens with excessive probability deviation from gradient computation, constraining gradient updates in a stable region.
\end{itemize}
Integrating these techniques yields the following objective function for IcePop:
\begin{equation}
\begin{aligned}
\mathcal{J}_{{\text{IcePop}}}(\theta) &= \mathbb{E}_{x \sim \mathcal{D}, \{y_i\}_{i=1}^G \sim \pi_{\textcolor{teal}{\text{infer}}}(\cdot \mid x; \theta_{\rm old})} \left[ \frac{1}{G} \sum_{i=1}^G \frac{1}{|y_i|} \sum_{t=1}^{|y_i|} \Big[\mathcal{M}\Bigl(\frac{\pi_{\textcolor{blue}{\text{train}}}(y_{i,t} \mid x, y_{i,<t};\theta_{\text{old}})}{\pi_{\textcolor{teal}{\text{infer}}}(y_{i,t} \mid x, y_{i,<t}; \theta_{\mathrm{old}})}; \alpha, \beta\Bigr) \right. \\ &\left. \qquad \qquad \qquad \qquad \quad \qquad \cdot \min \left( r_{i,t}\widehat{A}_{i,t}, \text{clip} \left( r_{i,t}, 1 - \varepsilon, 1 + \varepsilon \right) \widehat{A}_{i,t} \right)  - \gamma D_{\text{KL}}(\pi_{\theta}\|\pi_{\text{ref}})\right]\Bigg], &
\end{aligned}
\end{equation}
where $r_{i,t} = \frac{\pi_{\textcolor{blue}{\text{train}}}(y_{i,t} \mid x, y_{i,<t}; \ \theta)}{\pi_{\textcolor{blue}{\text{train}}}(y_{i,t} \mid x, y_{i,<t}; \ \theta_{\text{old}})}$, $\mathcal{M}(k)$ is the masking function defined as below:

\begin{equation}
\mathcal{M}(k) =\begin{cases} k & \text{if \ } k \in [\alpha, \beta], \\ 
0 & \text{otherwise}\end{cases}
\end{equation}
where $\alpha$,  $\beta$ controls the lower and upper limits.
Thus, the gradient of IcePop is
\begin{equation}
\label{eq:gradient_icepop}
\nabla_\theta \mathcal{J}_{\text{IcePop}}(\theta) \sim \mathbb{E}_{a \sim \textcolor{teal}{\pi_{\text{infer}}}(\theta_{\text{old}})} \Bigg[\mathcal{M}\Bigg(\frac{\textcolor{blue}{\pi_{\text{train}}}(a;\theta_{\text{old}})}{\textcolor{teal}{\pi_{\text{infer}}}(a;\theta_{\text{old}})}\Bigg ) \cdot \nabla_\theta \log \textcolor{blue}{\pi_{\text{train}}}(a;\theta) \cdot \hat{A} \cdot r(a)\Bigg)\Bigg].
\end{equation}
For $\dfrac{\textcolor{blue}{\pi_{\text{train}}}(a;\theta_{\text{old}})}{\textcolor{teal}{\pi_{\text{infer}}}(a;\theta_{\text{old}})} < \alpha$ and $\dfrac{\textcolor{blue}{\pi_{\text{train}}}(a;\theta_{\text{old}})}{\textcolor{teal}{\pi_{\text{infer}}}(a;\theta_{\text{old}})} > \beta$, IcePop discards all noisy gradients outside the region that may introduce potential instability into the training process. A more detailed introduction of IcePop could refer to our blog~\footnote{https://ringtech.notion.site/icepop}.

\subsubsection{C3PO++: Dynamically Partition Rollouts with Budget} 

\begin{figure}[!hbt]
    \centering
    \includegraphics[width=0.9\linewidth]{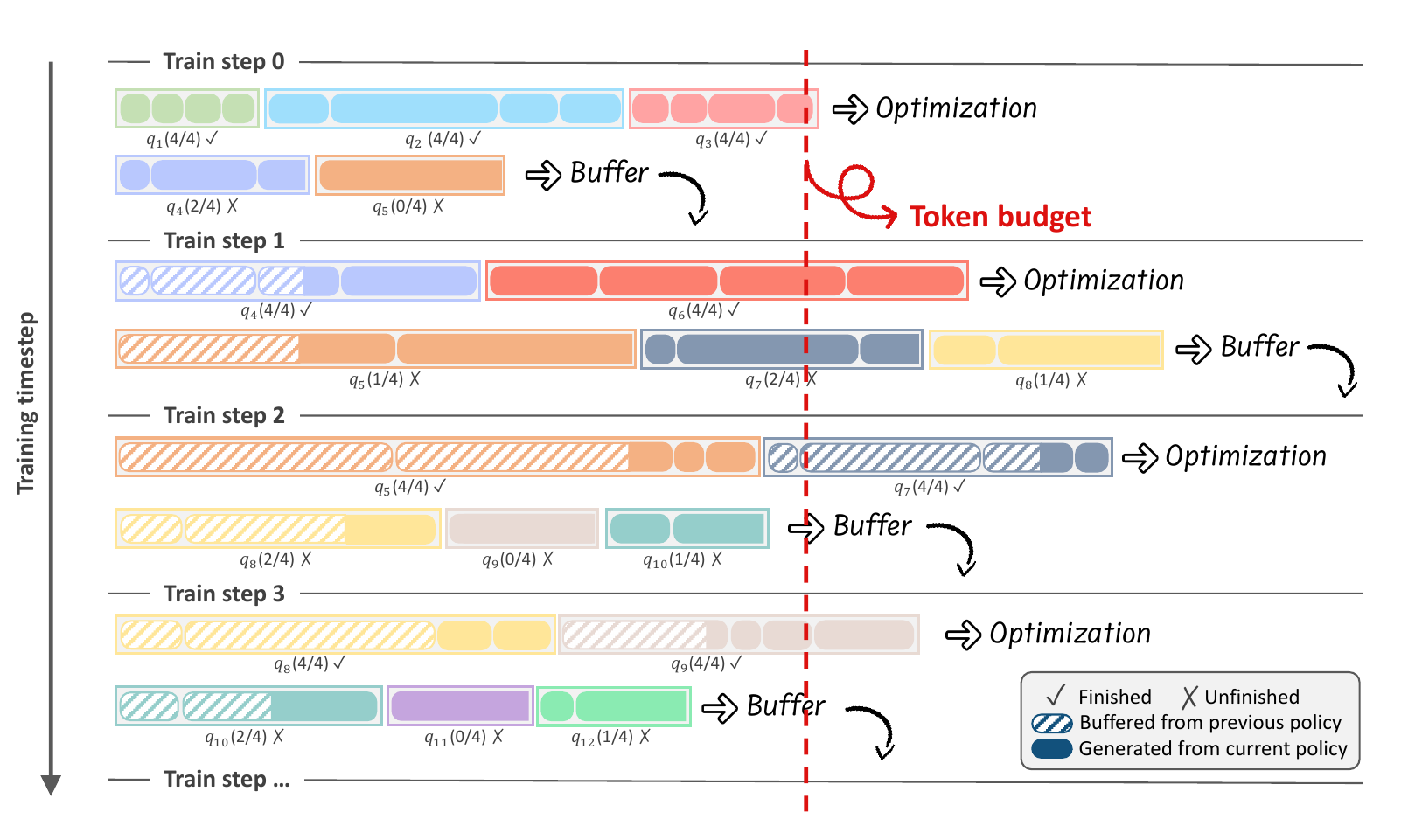}
    \caption{C3PO++ improves reinforcement learning efficiency for large thinking models by maintaining a rollout buffer across policy model versions. Once the rollout in an iteration reaches the token budget, optimization is performed; unfinished rollouts are stored in the buffer and resumed by the updated policy in the next iteration.}
    \label{fig:c3po++}
\end{figure}

We introduce C3PO++, an extension of C3PO~\citep{team2025ring} that incorporates a budget-controlled rollout partition mechanism. This approach dynamically partitions rollout generation to prevent idleness of computational resources caused by individual long rollouts. The system incorporates two modules: a high-throughput inference pool $P_\text{infer} $ with capacity $\Omega_{\text{infer}}$ for parallel generation, and a training pool $Q_\text{train}$ with capacity $\Omega_\text{train}$ for collecting completed trajectories. Similar to the idea of C3PO, we regulate the rollout generation with a \textit{token budget} ($\Phi$), which stabilizes the training updates and enables a highly efficient rollout process.

The C3PO++ procedure is detailed in Algorithm \ref{algo:c3po++}. At iteration $t$, the inference engine $\pi_{\text{infer};\theta_t}$ populates the inference pool by generating rollouts in parallel, while tracking the cumulative number of generated tokens $C$ in real-time. When a rollout reaches a terminal state (i.e., $\texttt{[EOS]}$), it will be moved from $P_{\text{infer}}$ to the training pool $Q_{\text{train}}$ and counted towards the training tokens $C$. Inference proceeds until $C$ reaches the token budget $\Phi$. At this point, the training engine $\pi_{\text{train};\theta_t}$ updates the parameters with the completed trajectories in $Q_{\text{train}}$ regulated by the token budget, which may include samples resumed from earlier inference versions. We denote the number of partitions a sequence has undergone as the retention period. For each iteration, the retention period of unfinished rollouts will be automatically increased by 1. Before each iteration, rollouts whose retention period exceeds a threshold $\sigma$ are purged from $P_{\text{infer}}$. Meanwhile, new prompts may be sampled to refill $P_{\text{infer}}$ until it reaches capacity $\Omega_{\text{infer}}$. After the model parameters are updated to $\theta_{t+1}$, the inference engine $\pi_{\text{infer};\theta_{t+1}}$ initiates a new iteration of rollout generation, continuing the process for rollouts within the valid retention period and monitored by the token budget.

\begin{algorithm}[!tbh]
\caption{C3PO++}
\label{algo:c3po++}
\KwIn{
initial parameters $\theta_0$; inference engine $\pi_{\text{infer};\theta}$; training engine $\pi_{\text{train};\theta}$;
token budget $\Phi$;
inference pool capacity $\Omega_{\text{infer}}$;
retention threshold $\sigma$.
}
\KwOut{Sequence of parameter updates $\theta_0 \rightarrow \theta_1 \rightarrow \cdots$}
\textbf{State:} Inference pool $P_{\text{infer}}$ (capacity $\Omega_{\text{infer}}$); training pool $Q_{\text{train}}$.
\Begin{
  $P_{\text{infer}} \gets \emptyset$; $Q_{\text{train}} \gets \emptyset$; $t \gets 0$ \;
  \While{not converged}{
    $C \gets 0$\;
    \ForPar{$o \in P_{\text{infer}}$}{
        \If{$\text{retention}(o, \sigma)$}{
          $P_{\text{infer}} \gets P_{\text{infer}} \backslash \{o\}$  \tcp*{remove overextended rollouts from inference pool}
        }
      }
    \While{$(C < \Phi)$}{
      \If{$|P_{\text{infer}}| < \Omega_{\text{infer}}$}{
        $P_{\text{infer}} \gets P_{\text{infer}} \cup \text{sample\_prompt}()$ \tcp*{maintain a full inference pool}
      }
      \ForPar{$o \in P_{\text{infer}}$}{
        $o \gets \pi_{\text{infer};\theta_t}(o)$ \tcp*{generate the next token for rollouts in parallel}
        \If{$\text{terminal}(o)$}{
          $C \gets C + |o|$ \tcp*{cumulate token amount}
          $Q_{\text{train}} \gets Q_{\text{train}} \cup \{o\}$ \tcp*{save completed rollouts for training}
          $P_{\text{infer}} \gets P_{\text{infer}} \backslash \{o\}$  \tcp*{remove completed rollouts from inference}
        }
      }
    }
    $\theta_{t+1} \gets \text{Update}(\theta_t, Q_{\text{train}})$ \;
    $Q_{\text{train}} \gets \emptyset$ \;
    $t \gets t + 1$ \;
  }
  \Return{$\theta_t$}
}
\end{algorithm}

\subsubsection{Training Recipe}
\label{subsec:rl-train}
All policy optimization was conducted using the ASystem framework. We employ the AdamW optimizer with hyperparameters $\beta_1$ = 0.9, $\beta_2$ = 0.999, weight decay of $0.01$, and with the MoE router bias held fixed. 
\textbf{For the reasoning RL stage}, we implemented the proposed IcePop ($\alpha=0.5$, $\beta=5$) and C3PO++ algorithms. The training configuration used a learning rate of $2 \times 10^{-6}$, a KL coefficient of $0.0$, and a sampling temperature of $1.0$. Each training step utilized 480 unique prompts, with 8 rollouts sampled per prompt and a maximum length of 65,536 tokens. 
\textbf{For the general RL stage}, we utilized GRPO with a learning rate of $3 \times 10^{-6}$, a KL coefficient of $0.0$, and a sampling temperature of $1.0$. Each step in this stage consisted of 80 unique questions with 8 outputs each, and a maximum length of 32,768 tokens.
\subsection{Experiments and Analysis}
\label{subsec:rl-exp}
This section presents experiments to validate the effectiveness of our proposed methods: IcePop, which ensures stable policy optimization, and C3PO++, which enables efficient rollout generation.
\subsubsection{IcePop}\label{subsec:rl-exp-icepop}

\paragraph{Setup} To evaluate the effectiveness of IcePop, we conduct preliminary experiments on the Ring-mini-2.0\footnote{\href{https://huggingface.co/inclusionAI/Ring-mini-2.0}{https://huggingface.co/inclusionAI/Ring-mini-2.0}} model, which is a MoE model with 16.8B total parameters and 0.75B activated parameters. We compare three settings: (1) IcePop with $\alpha=0.5, \beta=5$, (2) TIS~\citep{yao2025offpolicy} with the officially recommended setting, which mitigates the training-inference mismatch issue with importance-sampling correction, and (3) Vanilla GRPO without the KL-term. For fair comparison, we use the same training dataset for all models.

\paragraph{Preliminary results on Ring-mini-2.0.} As shown in Figure \ref{fig:icepop_aime25}, we can see that IcePop consistently outperforms TIS on the challenging benchmark AIME25, with a large gain along the training process, and finally improves the base score (63\%) by over 14\%, and expands the performance gap with TIS by relative 6\%.

\begin{figure}[!htb]
    \centering
    \includegraphics[width=0.6\linewidth]{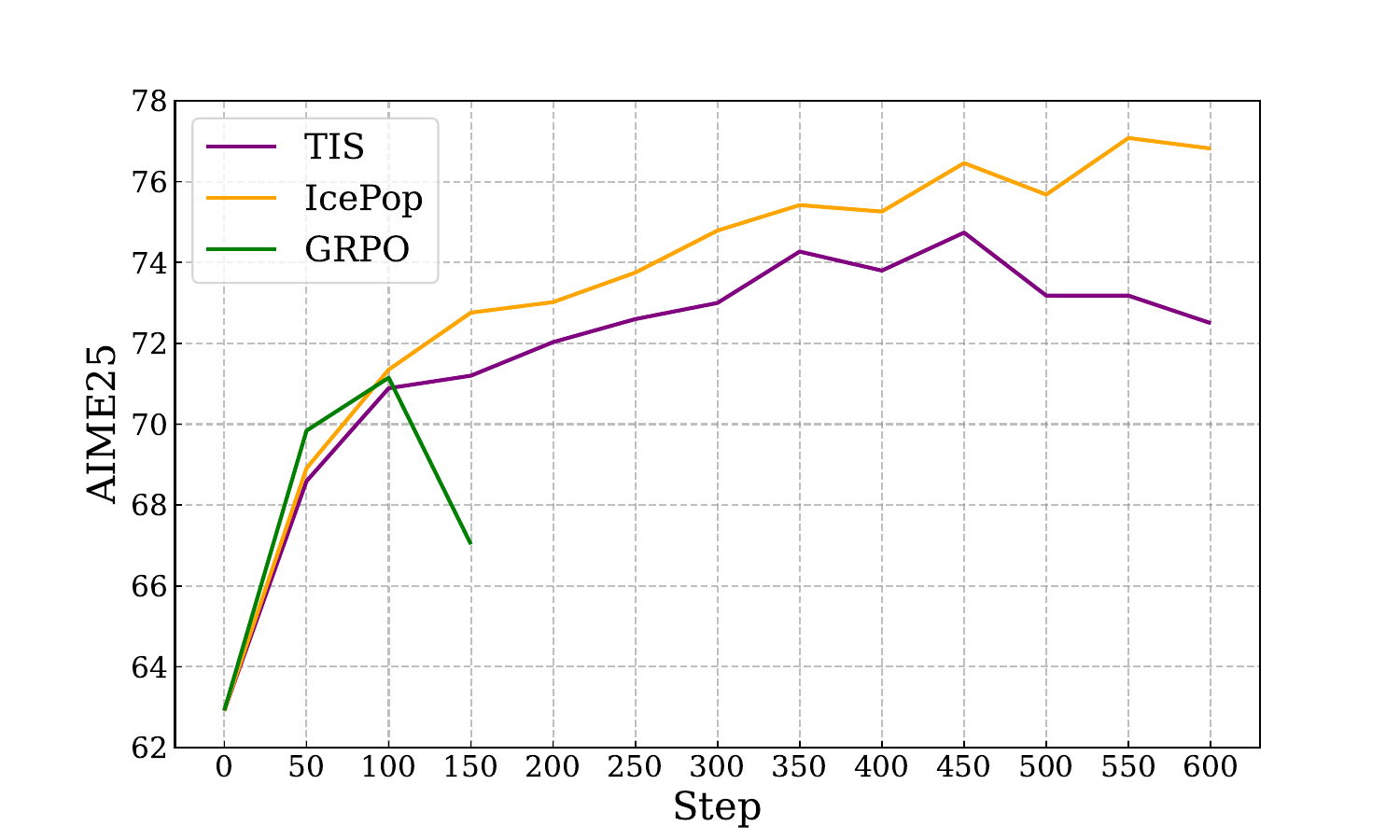}
    \caption{The performance comparison on AIME25 (Avg@64). We evaluate all models using the same setting.}
    \label{fig:icepop_aime25}
\end{figure}

\paragraph{Experiments on Ring-1T.} As training progresses, we can see from Figure~\ref{fig:icepop_ring_1t} that the original GRPO suffers from training instability, as both the gradient norms and the probability discrepancy between the inference and training engines tend to increase rapidly. However, after applying IcePop, we can observe that the mismatch issue has been largely mitigated, stabilizing the RL training process.

\begin{figure}[h!]
\centering
\begin{subfigure}[b]{0.45\textwidth}
\centering
\includegraphics[width=\textwidth]{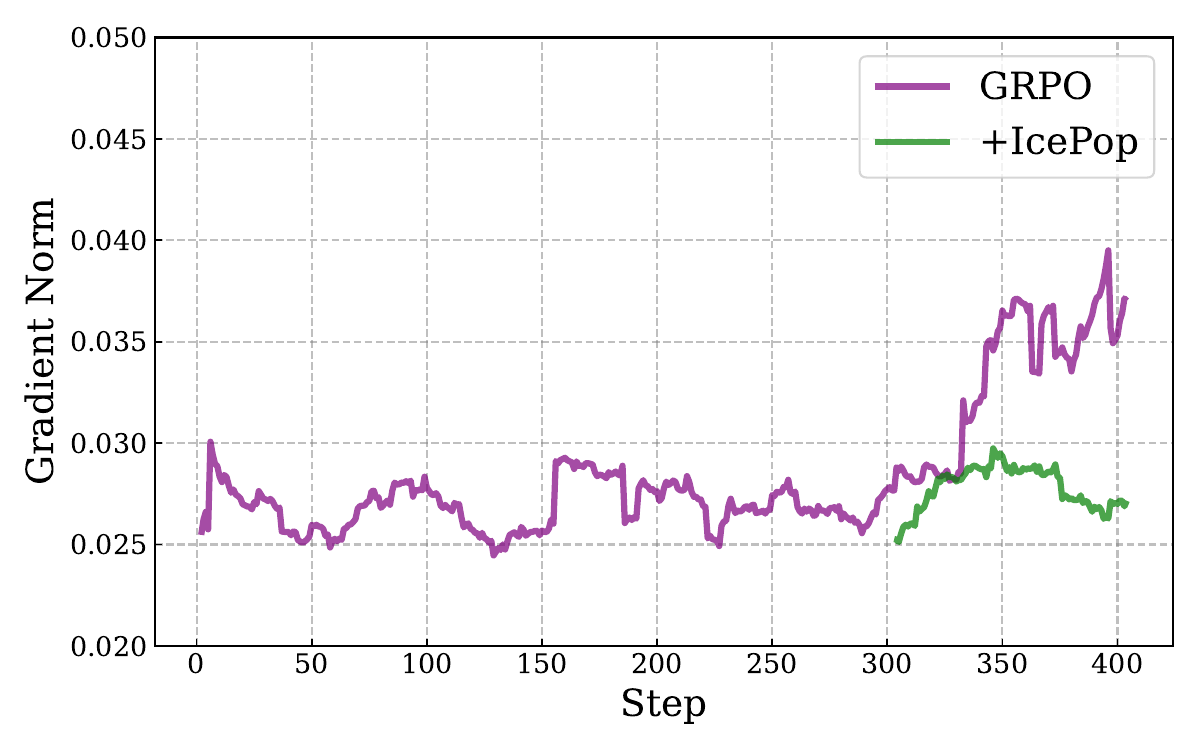}
\end{subfigure}
\hspace{2mm}
\begin{subfigure}[b]{0.45\textwidth}
\centering
\includegraphics[width=\textwidth]{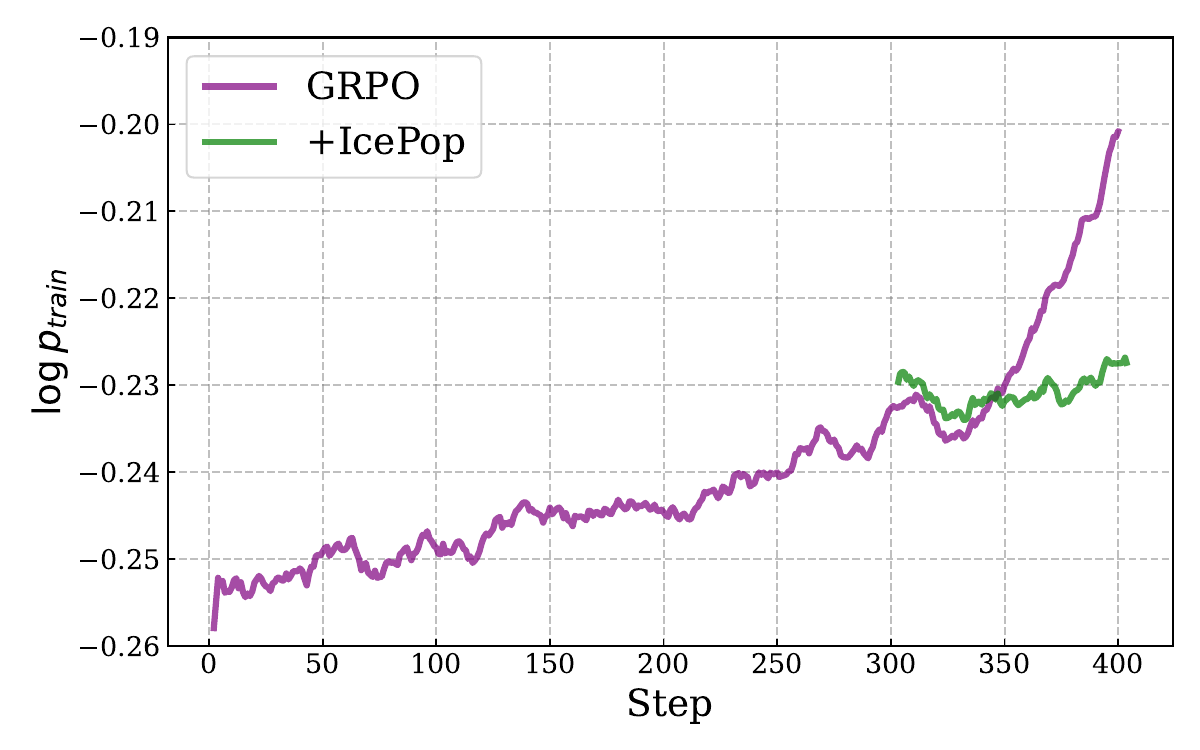}
\end{subfigure}
\begin{subfigure}[b]{0.45\textwidth}
\centering
\includegraphics[width=\textwidth]{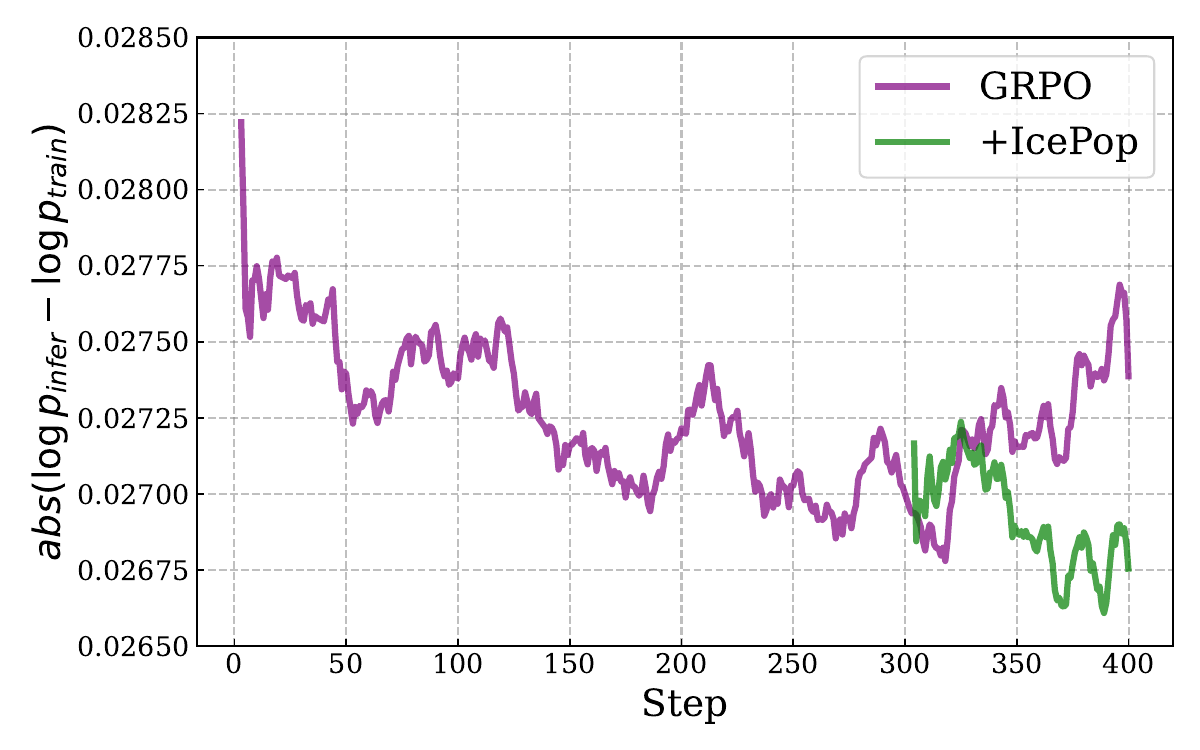}
\end{subfigure}
\hspace{2mm}
\begin{subfigure}[b]{0.45\textwidth}
\centering
\includegraphics[width=\textwidth]{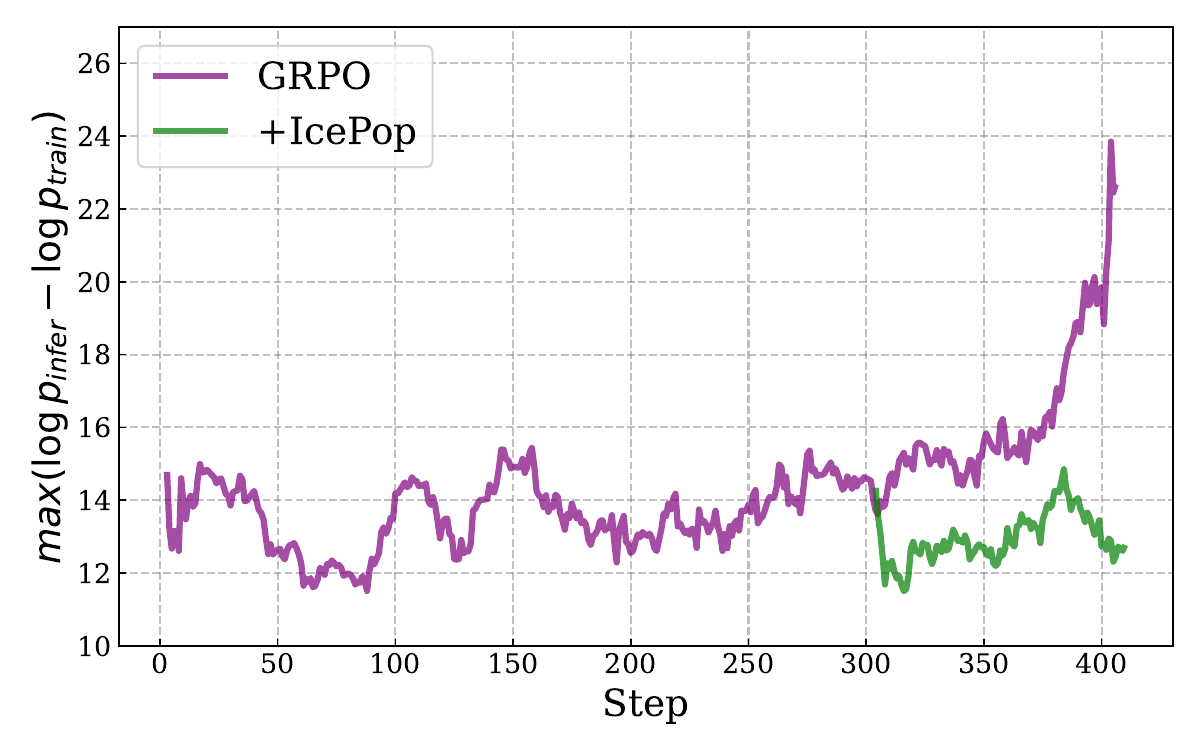}
\end{subfigure}
\caption{The training dynamics before and after applying IcePop.}
\label{fig:icepop_ring_1t}
\end{figure}

\subsubsection{C3PO++}\label{subsec:rl-exp-c3poplus}
We compare C3PO++ with the baseline setting that omits our budget-controlled rollout partition mechanism, assessing training efficiency and effectiveness in terms of training time, training reward, and benchmark performance. 
\begin{itemize}
    \item \textbf{Training Time.}~~As illustrated in Figure~\ref{fig:c3popp_time}, C3PO++ substantially reduces the time of the rollout phase, achieving an approximately 2.5 times speedup per step. Since rollout duration usually accounts for a large portion of training time in RL, the training optimization designed by C3PO++ yields about a 1.5 times speedup for the end-to-end phase per step, significantly boosting the training efficiency for reinforcement learning.
    \begin{figure}[!htb]
    \centering
    \begin{subfigure}[b]{0.45\textwidth}
    \centering
    \includegraphics[width=\textwidth]{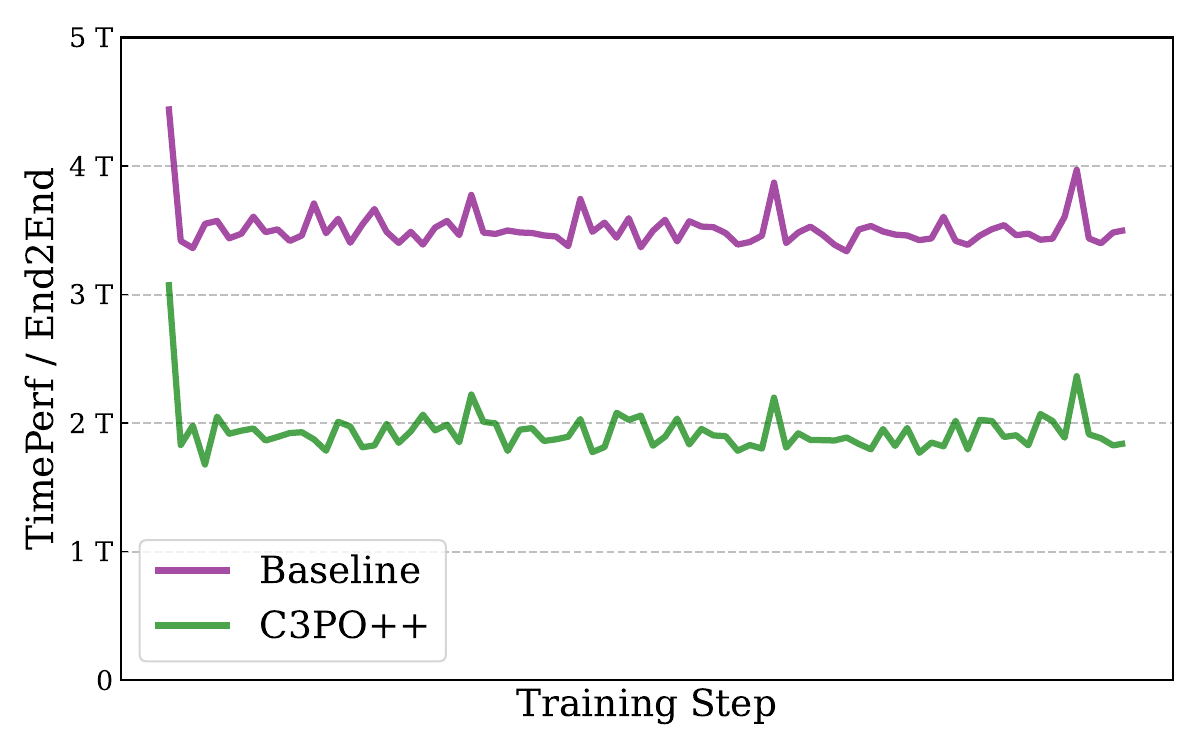}
    \end{subfigure}
    \hspace{2mm}
    \begin{subfigure}[b]{0.45\textwidth}
    \centering
    \includegraphics[width=\textwidth]{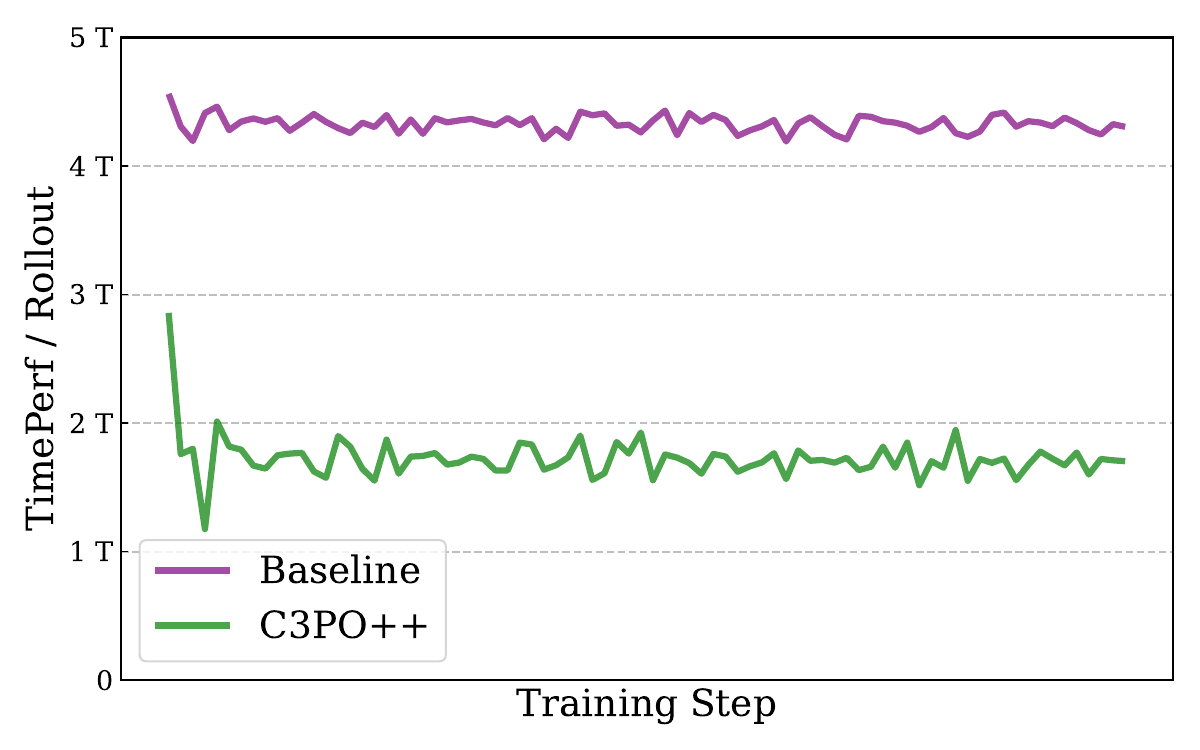}
    \end{subfigure}
    \caption{Comparison of time cost between C3PO++ and the baseline.}
   \label{fig:c3popp_time}
   \end{figure}
   
    \item \textbf{Reward and Performance.}~~As shown in Figure~\ref{fig:c3popp_reward}, the reward curve of C3PO++ remains close to that of the baseline, suggesting that our optimization in rollout management maintains comparable training dynamics in the reinforcement learning process. On the representative reasoning benchmarks, C3PO++ achieves performance on par with the baseline, demonstrating its strength in producing competitive results.
    \begin{figure}[!htb]
    \centering
    \begin{subfigure}[b]{0.45\textwidth}
    \centering
    \includegraphics[width=\textwidth]{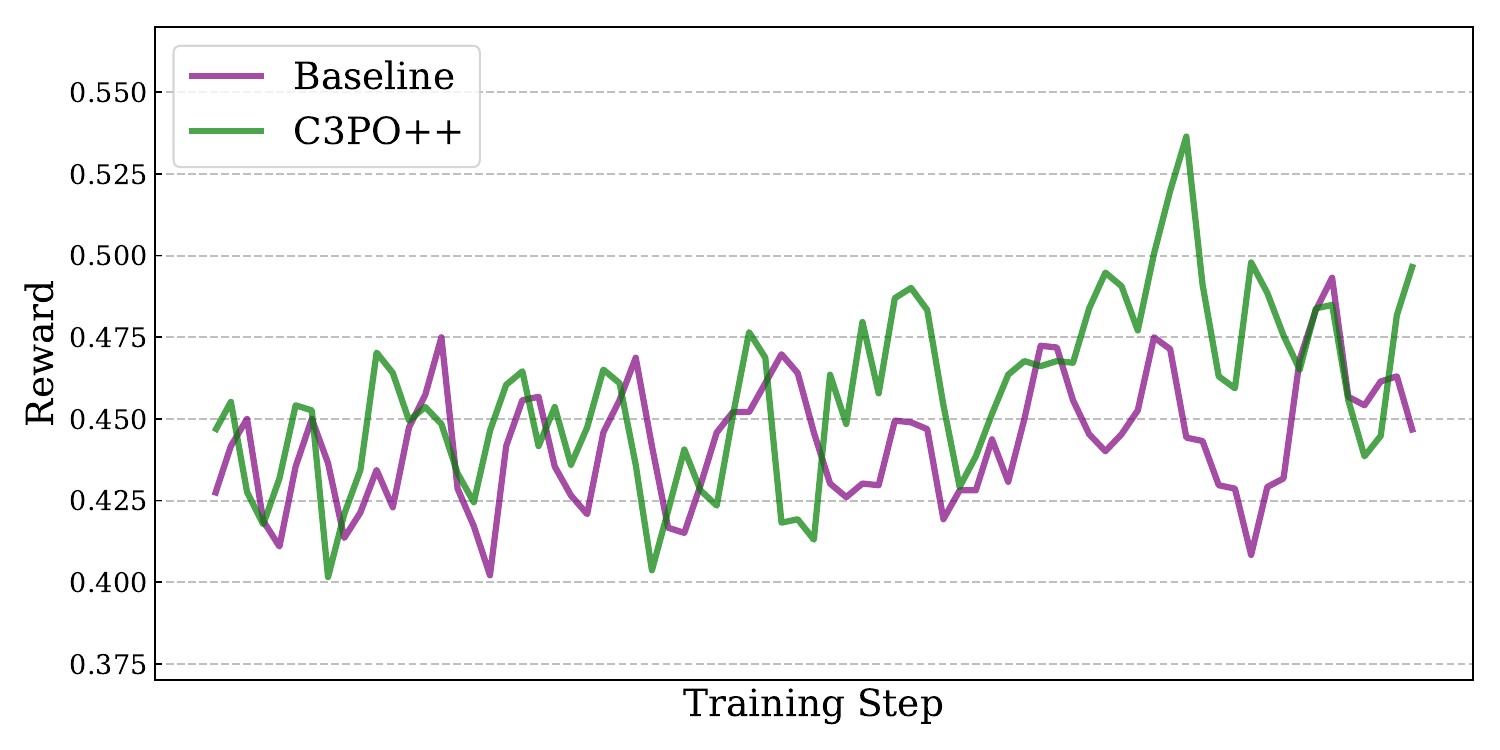}
    \end{subfigure}
    \hspace{2mm}
    \begin{subfigure}[b]{0.45\textwidth}
    \centering
    \includegraphics[width=\textwidth]{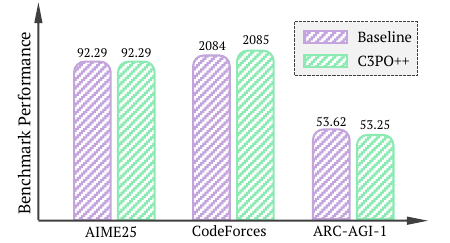}
    \end{subfigure}
    \caption{Comparison of reward and benchmark performance between C3PO++ and the baseline.}
   \label{fig:c3popp_reward}
   \end{figure}

\end{itemize}

\subsection{Large-Scale RL Infrastructure: ASystem}
\label{subsec:asys}

\begin{figure}[!hbt]
    \centering
    \includegraphics[width=0.85\linewidth]{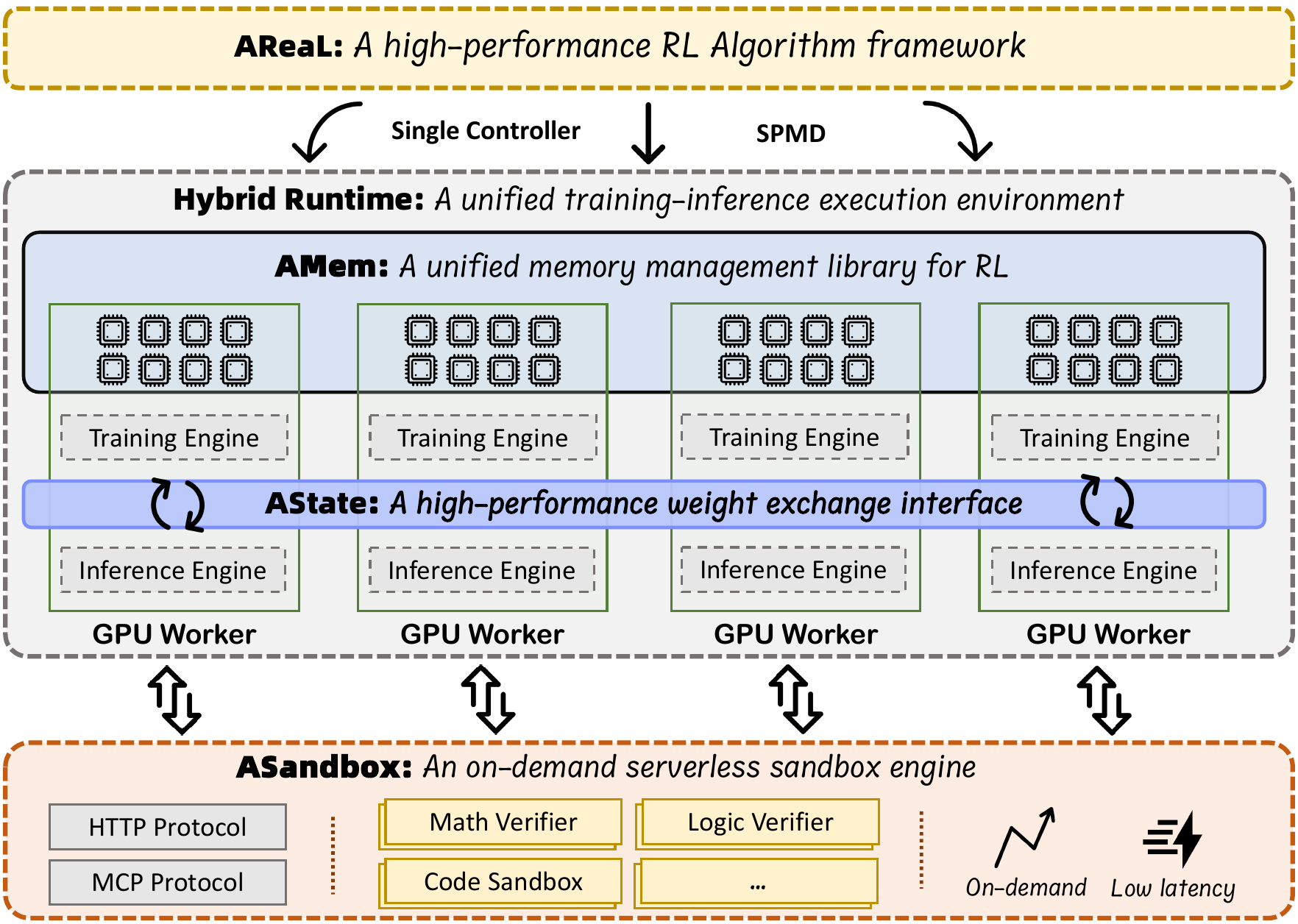}
    \caption{An overview of ASystem RL training framework.}
    \label{fig:asystem}
\end{figure}

Training \model{} with reinforcement learning requires a specialized infrastructure that can manage its unprecedented scale. The sheer size of the model, coupled with the inherent complexity of distributed RL workflows, poses unique challenges in memory management, state synchronization, and computational throughput. To this end, we developed ASystem, a high-performance RL framework whose components are co-designed with the requirements of \model{} in mind.

As illustrated in Figure~\ref{fig:asystem}, ASystem's architecture is built around a unified execution environment and includes the following key components, each engineered to address a specific bottleneck in the RL training for a trillion-parameter model:

\begin{itemize}
\item \textbf{Hybrid Runtime}: The core of ASystem, this runtime seamlessly integrates training and inference workloads. For \model{}, this means we can conduct massive parallel policy evaluation (inference) and model weight updates (training), eliminating the overhead of data transfer between separate systems and ensuring efficient utilization of thousands of GPUs.
\item \textbf{AMem}: AMem is a GPU memory management library designed to overcome the critical memory bottleneck in large-scale RL training, like that of our 1T model. It optimizes memory usage and data transfer, enabling larger batches, fewer OOM errors, and faster deployment with minimal code changes and no loss of accuracy.
\item \textbf{AState}: AState is a high-performance weight synchronization framework for RL. It efficiently addresses the challenge of distributing updated model parameters from trainers to inference actors using a zero-redundancy peer-to-peer mechanism, enabling synchronization of trillion-parameter models in under 10 seconds.
\item \textbf{ASandbox}:  A serverless environment for rapid scenario validation. By offering millisecond-scale cold start and high-throughput isolation, ASandbox accelerates evaluation of \model{} rollouts during large-scale RL training.
\end{itemize}

This foundational design, based on a SingleController + SPMD (Single Program, Multiple Data) architecture, delivers significant advantages for robust large-scale training. It provides plug-and-play support for training, inference, and reward model backends, facilitating independent debugging and development at scale. Crucially, by separating the control flow from the data flow, ASystem effectively mitigates the single-point data flow bottlenecks prevalent in mainstream SingleController frameworks. Furthermore, the system incorporates mechanisms for fast-fail reporting and automatic recovery from slow training and hangs, thereby enhancing overall training stability and efficiency for demanding workloads like our \model{} model.

\subsubsection{Hybrid Runtime: A Unified Training-Inference Execution Environment}

Hybrid Runtime is an integrated training-inference system designed for large-scale LLM reinforcement learning. It provides a high-performance, elastic, and scalable foundation by unifying efficient resource scheduling, linear scalability, comprehensive parallelism strategies, and a unified execution engine. The system is architected to support models of diverse architectures, scales, and training paradigms on large-scale clusters.

To bridge the gap between dynamic training and real-time inference in reinforcement learning (RL), we introduce \textbf{AState}, a high-speed framework for synchronizing weights between training and inference. AState provides a unified weight management API that supports diverse model architectures, deployment topologies, and pipeline paradigms without requiring framework modifications. At its core, a zero-redundancy peer-to-peer transmission mechanism delivers only necessary weight shards, enabling in-place updates on inference engines to eliminate costly data copies. This is complemented by a hardware–software co-design that optimizes data movement through NUMA topology and CPU-GPU affinity awareness, alongside a multi-transport communication layer (integrating RDMA, NCCL, and shared memory) that dynamically selects the optimal protocol based on data size and hardware topology. Consequently, AState achieves sub-second parameter updates, ensuring inference rollouts use the latest model and maintaining the critical training-inference alignment essential for stable policy optimization.

To enhance GPU memory efficiency, we introduce \textbf{AMem}, a memory and data transfer library optimized for RL workloads on GPU clusters. AMem enhances memory management efficiency through three key mechanisms: (1) Memory Switching, for the transparent release and resumption of training state, including NCCL communications and CUDA graphs; (2) Distributed Multi-path Transfer~\cite{shen2025flexlinkboostingnvlinkbandwidth}, which aggregates bandwidth across multiple channels; and (3) Unified Memory Pooling, for dynamic allocation across GPUs and nodes. By enabling larger batch sizes, reducing out-of-memory (OOM) errors, and accelerating system startup, AMem alleviates common bottlenecks in large-scale RL. The library is designed for transparency, requiring no model modifications and ensuring no impact on RL convergence, thereby providing robust infrastructure support for the Hybrid Runtime and AState components.

\subsubsection{ASandbox: An On-Demand Serverless Sandbox Engine}
ASandbox is a serverless sandbox engine for RL, providing rapid, isolated environments for tasks like code execution and terminal simulation. Integrated with Kubernetes and deployable as a standalone FaaS cluster, it executes RL tasks via function calls. It offers specialized sandboxes (e.g., math, code, STEM, terminal) supporting HTTP and MCP protocols. To ensure the consistent, stable feedback critical for RL training, it features: 1) Security: Kernel-level isolation via secure containers (runsc, kata); 2) Availability: Automatic node failure detection and isolation; 3) Speed: 100ms startup via image caching, cgroups, and fork; 4) Scalability: 5,000 QPS/200ms throughput via scheduling partitions.

\subsubsection{AReaL: A High-Performance RL Algorithm Framework}
ASystem is a unified, high-performance foundation for distributed reinforcement learning. Its reinforcement learning component, AReaL, is an open-source framework~\citep{fu2025areal} built to prioritize algorithm development by balancing ease of use with system flexibility. It offers both single-controller and SPMD interfaces through minimalist APIs and an extensible plugin mechanism, allowing researchers to focus on algorithmic innovation.

AReaL is characterized by several key features as follows:

\begin{itemize}
    \item \textbf{Asynchronous Multi-Stage Pipeline:} A fully decoupled architecture that concurrently executes trajectory generation, reward computation, and training. This overlap eliminates rollout long-tail issues and maximizes hardware utilization.
    \item \textbf{Efficient Data Management:} Intelligent data packing and sharding minimize padding and rebalancing overhead, reducing computational waste and training stalls.
    \item \textbf{Fault Tolerance:} The system features automated error detection, retry, and recovery mechanisms to ensure stability amidst hardware and software failures.
    \item \textbf{Massive Scalability:} By separating control and data planes, AReaL avoids the single-controller bottleneck, enabling seamless scaling across large clusters.
\end{itemize}

\section{Evaluation}
\label{sec:eval}
This section presents the performance of our \model{} on a suite of challenging benchmarks spanning mathematics, coding, and logical reasoning, as well as other general tasks, comparing it against leading reasoning models.

\subsection{Benchmarks}
\label{subsec:bench}

To comprehensively assess \model{}, we conduct evaluations across a wide range of benchmarks, primarily covering 8 domains: knowledge, coding, math, reasoning, alignment, healthcare, multi-turn, and agent.

\begin{itemize}
    \item \textbf{Knowledge:} GPQA-Diamond~\citep{gpqa}, MMLU-Pro~\citep{mmlu-pro}, C-Eval~\citep{ceval}, Phybench~\citep{qiu2025phybenchholisticevaluationphysical}, AGIEval~\citep{zhong2023agieval}, TriviaQA~\citep{2017arXivtriviaqa}, CMMLU~\citep{li2023cmmlu}.
    \item \textbf{Coding:} LiveCodeBench-v6 (2408 to 2505)~\citep{livecodebench}, CodeForces~\footnote{The Codeforces was assessed through problems from 14 Div. 2 contests of Codeforces, combined with expert-designed test cases, followed by the computation of expected ratings and competitor proportions. It is worth noting that the highest rating attainable is 2209.}, Aider~\footnote{https://aider.chat/docs/benchmarks.html\#the-benchmark}.
    \item \textbf{Math:} AIME 2025~\citep{aime}, Omni-MATH~\citep{gao2024omnimathuniversalolympiadlevel}, HMMT 2025, CNMO 2024, FinanceReasoning~\citep{tang2025financereasoning}, UGMathBench~\citep{xu2025ugmathbench}.
    \item \textbf{Reasoning:} ARC-AGI-1~\citep{chollet2024arc}, BBEH~\citep{kazemi2025big}, ZebraLogic~\citep{lin2025zebralogic}, HLE~\citep{phan2025humanity}.
    \item \textbf{Alignment:} ArenaHard v2~\citep{li2024live}, Creative Writing v3~\citep{creative-writing-bench-v3}, IFEval~\citep{zhou2023instruction}.
    \item \textbf{Healthcare:} HealthBench~\citep{HealthBench}.
    \item \textbf{Multi-turn:} MultiChallenge~\citep{sirdeshmukh2025multichallenge}.
    \item \textbf{Agent:} BFCL v3~\citep{BFCL}.
\end{itemize}

We benchmark \model{} against leading open-weights models (DeepSeek-V3.1-Terminus-Thinking, and Qwen-35B-A22B-Thinking-2507) and proprietary API models (Gemini-2.5-pro, GPT-5-Thinking). All evaluations use controlled experimental conditions with standardized configurations.

\definecolor{TABLE_LINE}{gray}{0.95} 
\definecolor{ringblue}{RGB}{65, 155, 255}

\begin{table}[htbp]
\centering
\caption{Performance comparison across multiple benchmarks.}
\label{tab:main_results}
\resizebox{\linewidth}{!}{
\begin{tabular}{lccccc}
\toprule
\multicolumn{1}{l|}{\multirow{3}{*}{\textbf{Benchmark}}} & \multicolumn{3}{c}{\textbf{Open Weights}} & \multicolumn{2}{c}{\textbf{Close Weights}} \\
\cmidrule(lr){2-4} \cmidrule(lr){5-6}
\multicolumn{1}{l|}{} & 
\multirow{2}{*}{\centering\textcolor[HTML]{0369ff}{\textbf{\model{}}}} &  
\textbf{DeepSeek-V3.1-} & 
\multicolumn{1}{c}{\textbf{Qwen3-235B-A22-}} & 
\textbf{Gemini-} & 
\textbf{GPT-5-} \\
\multicolumn{1}{l|}{} & 
& 
\textbf{Terminus-Thinking} & 
\multicolumn{1}{c}{\textbf{Thinking-2507}} & 
\textbf{2.5-Pro} & 
\textbf{Thinking (High)} \\
\midrule
\multicolumn{1}{l|}{Architecture}  & MoE & MoE & MoE & - & - \\
\multicolumn{1}{l|}{\# Total Params}  & 1T & 671B & 235B & - & - \\
\multicolumn{1}{l|}{\# Activated Params}  & 50B & 37B & 22B & - & - \\ \midrule

\rowcolor{TABLE_LINE} \multicolumn{6}{c}{\emph{Math}} \\
\multicolumn{1}{l|}{AIME 2025 {\footnotesize (Avg@64)}} & \textcolor{ringblue}{\second{93.40}} & 89.06 & 92.30 & 88.00 & \best{94.60} \\
\multicolumn{1}{l|}{Omni-MATH} & \textcolor{ringblue}{\second{82.63}} & 81.93 & 82.52 & 82.14 & \best{82.90} \\
\multicolumn{1}{l|}{HMMT25 {\footnotesize (Avg@16)}} & \textcolor{ringblue}{\second{86.72}} & 86.10 & 83.90 & 82.50 & \best{93.30} \\
\multicolumn{1}{l|}{CNMO 2024} & \second{88.54} & 85.42 & \best{89.50} & 80.64 & 87.93 \\
\multicolumn{1}{l|}{FinanceReasoning} & 87.42 & \second{87.76} & 87.65 & 87.33 & \best{89.33} \\
\multicolumn{1}{l|}{UGMathBench} & 76.47 & \second{77.19} & 77.00 & 74.50 & \best{80.18} \\

\rowcolor{TABLE_LINE} \multicolumn{6}{c}{\emph{Coding}} \\
\multicolumn{1}{l|}{LCB-v6 (2408-2505) {\footnotesize (Avg@4)}} & \textcolor{ringblue}{\second{78.30}} & 75.33 & 75.72 & 70.65 & \best{80.60} \\
\multicolumn{1}{l|}{CodeForces {\footnotesize (rating)}} & \textcolor{ringblue}{\best{2088}} & \second{2073} & 2055 & 1837 & 1918 \\
\multicolumn{1}{l|}{CodeForces {\footnotesize (percentile)}} & \textcolor{ringblue}{\best{97.85}} & \second{97.76} & 97.55 & 93.47 & 86.18 \\
\multicolumn{1}{l|}{Aider} & 78.57 & 92.86 & 88.91 & \second{94.36} & \best{95.49} \\

\rowcolor{TABLE_LINE} \multicolumn{6}{c}{\emph{Reasoning}} \\
\multicolumn{1}{l|}{ARC-AGI-1} & \textcolor{ringblue}{\second{55.94}} & 40.62 & 48.12 & 45.44 & \best{65.70} \\  
\multicolumn{1}{l|}{BBEH} & 59.63 & \second{61.04} & 60.00 & 51.51 & \best{72.78} \\
\multicolumn{1}{l|}{ZebraLogic} & 95.15 & 96.33 & \second{97.03} & 92.40\textsuperscript{†} & \best{98.00} \\
\multicolumn{1}{l|}{HLE} & 16.03 & 17.82 & 13.95 & \second{21.60\textsuperscript{†}} & \best{24.84} \\

\rowcolor{TABLE_LINE} \multicolumn{6}{c}{\emph{Knowledge}} \\
\multicolumn{1}{l|}{GPQA-Diamond} & 78.63 & 81.00\textsuperscript{†} & 81.10\textsuperscript{†} & \best{86.40\textsuperscript{†}} & \second{86.05}  \\
\multicolumn{1}{l|}{MMLU-Pro}  & 80.54  & 85.00 \textsuperscript{†} & 84.40\textsuperscript{†} & \second{85.62} & \best{86.21}      \\ 
\multicolumn{1}{l|}{C-Eval} & \second{91.53} & 91.22 & \best{93.13} & 90.14 & 88.27 \\ 
\multicolumn{1}{l|}{Phybench} & 42.65 & 47.91 & 42.61 & \best{55.01} & \second{48.53} \\
\multicolumn{1}{l|}{AGIEval} & 88.13 & \second{89.83} & \best{90.01} & 88.99 & 88.56 \\
\multicolumn{1}{l|}{TriviaQA} & 78.59 & 82.77 & 79.63 & \second{84.45} & \best{86.32} \\
\multicolumn{1}{l|}{CMMLU} & 89.14 & \second{89.20} & \best{90.64} & 88.83 & 87.09 \\

\rowcolor{TABLE_LINE} \multicolumn{6}{c}{\emph{Alignment}} \\
\multicolumn{1}{l|}{ArenaHard v2 {\footnotesize (win-rate)}} & \textcolor{ringblue}{\second{81.59}} & 60.27 & 80.18 & 79.20 & \best{82.91} \\  
\multicolumn{1}{l|}{ArenaHard v2 {\footnotesize (Elo)}} & \textcolor{ringblue}{\best{84.52}} & 62.73 & 81.72 & 80.92 & \second{83.18} \\
\multicolumn{1}{l|}{Creative Writing v3} & 85.40 & 85.24 & 85.49 & \second{85.70} & \best{89.69} \\
\multicolumn{1}{l|}{IFEval} & 85.21 & \second{89.09} & 87.80 \textsuperscript{†} & 88.72 & \best{95.38} \\

\rowcolor{TABLE_LINE} \multicolumn{6}{c}{\emph{Healthcare}} \\
\multicolumn{1}{l|}{HealthBench} & \textcolor{ringblue}{\second{57.93}} & 50.19 & 55.56 & 49.39 & \best{67.20} \\

\rowcolor{TABLE_LINE} \multicolumn{6}{c}{\emph{Multi-turn}} \\
\multicolumn{1}{l|}{MultiChallenge} & 50.92 & 45.79 & 52.75 & \second{54.58} & \best{69.60\textsuperscript{†}} \\

\rowcolor{TABLE_LINE} \multicolumn{6}{c}{\emph{Agent}} \\
\multicolumn{1}{l|}{BFCL v3} & \second{68.82} & 62.01 & \best{73.53} & 61.36 & 57.21 \\

\midrule
\multicolumn{6}{@{}l@{}}{\small\textsuperscript{†} Results reported in official model documentation.} \\
\multicolumn{6}{@{}l@{}}{\small\textcolor{ringblue}{Blue} denotes \model{} achieving state-of-the-art among open-source models.} \\
\multicolumn{6}{@{}l@{}}{\small\textbf{Bold} indicates first place overall performance.} \\
\multicolumn{6}{@{}l@{}}{\small\underline{Underline} indicates second place overall performance.} \\

\bottomrule
\end{tabular}}

\end{table}

\subsection{Evaluation Settings}
All thinking models are evaluated under a standardized pipeline for a fair comparison. Benchmarks including AIME 2025, LiveCodeBench-v6, ARC-AGI-1, CodeForces, HMMT 2025, ZebraLogic, HLE, and BFCL v3 are assessed with a 128K context window, extended via YaRN \citep{peng2023yarn} for models with insufficient native context. Benchmarks including Aider, CNMO 2024, and BBEH use a 64K context window. Other benchmarks use a 32K context window. For baseline models, we use their official hyperparameters for open-weights models, while using vendor-recommended settings for proprietary APIs. For baselines with officially reported results, we report the higher score between our reproduction and the official release. 

\label{subsec:eval-set}

\subsection{Results}
\label{subsec:eval-res}

Table~\ref{tab:main_results} provides a comprehensive comparison of \model{} against leading thinking models. The following sections provide a detailed analysis of its performance across different aspects:

\paragraph{Mathematical Reasoning} 
\model{} demonstrates leading mathematical reasoning capabilities, as evidenced by its performance on challenging benchmarks. By relying solely on natural language reasoning, it achieves 93.40\% on AIME 25 and 86.72\% on HMMT 25, securing the second-highest rank overall and leading all open-weights models. Furthermore, the model delivers competitive results across specialized mathematical domains, scoring 82.63\% on Omni-MATH and 88.54\% on CNMO 2024. These results highlight a particular proficiency in complex, Olympiad-style problem-solving. This demonstrates that a stable and efficient RL training recipe, together with a diverse and high-quality math training dataset, drives superior mathematical reasoning across competition-level benchmarks. 

Moreover, we evaluate the mathematical reasoning capabilities of \model{} on the IMO 2025. Specifically, \model{} is integrated into the multi-agent framework AWorld~\citep{yu2025aworld} and tasked with solving the problems through pure natural language reasoning, without relying on code generation or external symbolic solvers. The model successfully solved Problems 1, 3, 4, and 5 on its first attempt, a performance corresponding to the IMO silver medal level. On its third attempt, it generated a nearly complete geometric proof for Problem 2. For the most challenging Problem 6, which no AI participant solved correctly during IMO 2025, \model{} converged to the same incorrect answer (4048) as Gemini 2.5 Pro, whereas the correct answer is 2112. A detailed case study is available in Appendix~\ref{sec:imo}. We believe the outstanding natural language reasoning ability will generalize to a broader range of tasks, paving the way for enhanced overall performance.

\paragraph{Coding Capabilities} 
As the results show, \model{} demonstrates exceptional performance in programming tasks that demand iterative refinement and deep logical reasoning, establishing a leading position among both open-weights and closed-weights models. On LiveCodeBench-v6 (2408-2505), it achieves a top score of 78.30\%, outperforming DeepSeek-V3.1 by 2.97 points and Qwen3-235B-A22B-Thinking-2507 by 2.58 percentage points. Furthermore, on CodeForces (rating), \model{} attains a score of 2088, which is the highest score among all models and exceeds the performance of both open-source competitors and closed-source APIs. It indicates that our carefully synthesized dataset shapes \model{}'s robust performance on programming applications, which forms a strong foundation for future endeavors on agentic applications. 

\paragraph{Logical Reasoning} 
\model{} demonstrates promising capabilities in other logical reasoning tasks. Powered by sourcing carefully selected logical games from multiple domains, \model achieves a score of 55.94\% on the challenging ARC-AGI-1 benchmark, ranking second overall. This performance places it only behind GPT-5-Thinking (65.70\%) and represents a substantial improvement of +15.32 percentage points over DeepSeek-V3.1 (40.62\%) and +7.82 points over Qwen3-235B-A22B-Thinking-2507 (48.12\%).

\paragraph{Human Alignment} In addition to the reasoning RL training, we also leverage a general RL training stage to equip the reasoning model with strong performance on general tasks. From Table~\ref{tab:main_results}, \model{} achieves strong alignment with human preferences in complex scenarios. On the ArenaHard v2 benchmark, it attains an 81.59\% win-rate, ranking second overall and trailing GPT-5-Thinking by only 1.32 percentage points. It also leads all models with an Elo rating of 84.52. In Creative Writing v3, \model{} scores 85.40\%, performing within 0.1 percentage points of the leading open-source model. These results confirm \model{}'s effectiveness in balancing human preference alignment with broad capabilities—a critical advantage for real-world deployment.

\paragraph{Healthcare Capabilities} 
On HealthBench, \model{} attains a score of 57.93\%, ranking second overall and leading the field of open-source models. This performance indicates proficient clinical knowledge integration and suggests the model's viability for complex healthcare tasks.

\section{Conclusion}
\label{sec:conclu}

In this work, we have presented \model{}, a landmark achievement as the first open-weights 1 \textbf{Trillion} parameter thinking model. This project successfully addressed the profound and unprecedented system and algorithmic challenges inherent in scaling reinforcement learning to a trillion-parameter regime. The core of our contribution lies in three interconnected innovations: the IcePop for resolving training-inference mismatches, the C3PO++ for efficient long-trajectory rollouts, and the ASystem framework that eliminates scalability bottlenecks and ensures training stability. Together, these advancements enabled the stable and efficient training of \model{}, which has demonstrated breakthrough, state-of-the-art performance across a rigorous set of benchmarks spanning mathematical reasoning, competitive programming, and general intelligence. These results validate our approach and systems, demonstrating that trillion-parameter reasoning models are not only feasible but also exhibit exceptional capability.

\section{Limitations \& Future Work}
Despite its landmark achievements, \model{} and its associated training systems have several limitations that point to fruitful directions for future research.

\begin{itemize}
    \item \textbf{Model Architecture \& Inference Efficiency:} The model's use of GQA~\citep{ainslie2023gqatraininggeneralizedmultiquery} provides a solid balance between performance and speed. However, for our \model{} thinking model, which generates extensive internal ``thought'' processes, the inference cost imposed by GQA remains non-trivial. Future work will therefore explore alternative mechanisms, such as MoBA~\citep{lu2025mobamixtureblockattention} or advanced linear attention variants, to achieve the higher throughput required for efficient inference. 

    \item \textbf{Training-Inference Consistency:} While our IcePop methodology mitigates the major training-inference mismatch, it does not achieve perfect training-inference consistency. Underlying numerical discrepancies between the training and inference computational operators persist as a latent source of instability. Resolving this fundamental systems challenge is imperative for the stable scaling of future models.
    
    \item \textbf{Capability Deficiencies:} The training strategy for \model{} was optimized for foundational natural language reasoning, leaving advanced agentic skills (e.g., tool use) under-optimized. Future iterations will position \model{} as a base for such capabilities, integrating specialized data and training paradigms like agentic RL to cultivate sophisticated autonomous problem-solving. Additionally, minor issues such as identity confusion and linguistic code-switching, attributed to data impurity and insufficient regularization, will be addressed through refined data curation techniques.
    
\end{itemize}

\clearpage
\section{Contributors}
\label{sec:contri}

\DTLnewdb{names}
\DTLnewrow{names} \DTLnewdbentry{names}{name}{Zihao Wang}
\DTLnewrow{names} \DTLnewdbentry{names}{name}{Kuan Xu}
\DTLnewrow{names} \DTLnewdbentry{names}{name}{Jia Guo}
\DTLnewrow{names} \DTLnewdbentry{names}{name}{Xin Zhao}
\DTLnewrow{names} \DTLnewdbentry{names}{name}{Yan Sun}
\DTLnewrow{names} \DTLnewdbentry{names}{name}{Liang Jiang}
\DTLnewrow{names} \DTLnewdbentry{names}{name}{Zhenyu Huang}
\DTLnewrow{names} \DTLnewdbentry{names}{name}{Shuaicheng Li}
\DTLnewrow{names} \DTLnewdbentry{names}{name}{Yongkang Liu}
\DTLnewrow{names} \DTLnewdbentry{names}{name}{Xiaopei Wan}
\DTLnewrow{names} \DTLnewdbentry{names}{name}{Yuchen Yan}
\DTLnewrow{names} \DTLnewdbentry{names}{name}{Xueyu Hu}
\DTLnewrow{names} \DTLnewdbentry{names}{name}{Xinyu Kong}
\DTLnewrow{names} \DTLnewdbentry{names}{name}{Zhenduo Zhang}
\DTLnewrow{names} \DTLnewdbentry{names}{name}{Qianggang Cao}
\DTLnewrow{names} \DTLnewdbentry{names}{name}{Zhixun Li}
\DTLnewrow{names} \DTLnewdbentry{names}{name}{Xinyu Tang}
\DTLnewrow{names} \DTLnewdbentry{names}{name}{Zujie Wen$^{\dagger}$}
\DTLnewrow{names} \DTLnewdbentry{names}{name}{Tongkai Yang}
\DTLnewrow{names} \DTLnewdbentry{names}{name}{Hao Dai}
\DTLnewrow{names} \DTLnewdbentry{names}{name}{Zhiqiang Ding}
\DTLnewrow{names} \DTLnewdbentry{names}{name}{Jun Mei}
\DTLnewrow{names} \DTLnewdbentry{names}{name}{Zhenxuan Pan}
\DTLnewrow{names} \DTLnewdbentry{names}{name}{Junping Zhao}
\DTLnewrow{names} \DTLnewdbentry{names}{name}{Jian Liu}
\DTLnewrow{names} \DTLnewdbentry{names}{name}{Shuwei Gu}
\DTLnewrow{names} \DTLnewdbentry{names}{name}{Xudong Han}
\DTLnewrow{names} \DTLnewdbentry{names}{name}{Hong Liu}
\DTLnewrow{names} \DTLnewdbentry{names}{name}{Le Su}
\DTLnewrow{names} \DTLnewdbentry{names}{name}{Senlin Zhu}
\DTLnewrow{names} \DTLnewdbentry{names}{name}{Chaokun Yang}
\DTLnewrow{names} \DTLnewdbentry{names}{name}{Guowei Wang}
\DTLnewrow{names} \DTLnewdbentry{names}{name}{Yuhong Guo}
\DTLnewrow{names} \DTLnewdbentry{names}{name}{Fagui Mao}
\DTLnewrow{names} \DTLnewdbentry{names}{name}{Tao Wu}
\DTLnewrow{names} \DTLnewdbentry{names}{name}{Jianxin Lai}
\DTLnewrow{names} \DTLnewdbentry{names}{name}{Deng Zhao}
\DTLnewrow{names} \DTLnewdbentry{names}{name}{Bin Hu}
\DTLnewrow{names} \DTLnewdbentry{names}{name}{Xiaodong Yan}
\DTLnewrow{names} \DTLnewdbentry{names}{name}{Cai Chen}
\DTLnewrow{names} \DTLnewdbentry{names}{name}{Liangcheng Fu}
\DTLnewrow{names} \DTLnewdbentry{names}{name}{Jiaming Liu}
\DTLnewrow{names} \DTLnewdbentry{names}{name}{Shaomian Zheng}
\DTLnewrow{names} \DTLnewdbentry{names}{name}{Wang Ren}
\DTLnewrow{names} \DTLnewdbentry{names}{name}{Qi Zuo}
\DTLnewrow{names} \DTLnewdbentry{names}{name}{Zhizhen Liu}
\DTLnewrow{names} \DTLnewdbentry{names}{name}{Qiang Cheng}
\DTLnewrow{names} \DTLnewdbentry{names}{name}{Donghai You}
\DTLnewrow{names} \DTLnewdbentry{names}{name}{Jiannan Shi}
\DTLnewrow{names} \DTLnewdbentry{names}{name}{Bin Jing}
\DTLnewrow{names} \DTLnewdbentry{names}{name}{Fanzhuang Meng}
\DTLnewrow{names} \DTLnewdbentry{names}{name}{Jin Yang}
\DTLnewrow{names} \DTLnewdbentry{names}{name}{Shaofei Wang}
\DTLnewrow{names} \DTLnewdbentry{names}{name}{Chengyao Wen}
\DTLnewrow{names} \DTLnewdbentry{names}{name}{Longfei Zheng}
\DTLnewrow{names} \DTLnewdbentry{names}{name}{Zhankai Xu}
\DTLnewrow{names} \DTLnewdbentry{names}{name}{Zhengke Gui}
\DTLnewrow{names} \DTLnewdbentry{names}{name}{Qitao Shi}
\DTLnewrow{names} \DTLnewdbentry{names}{name}{Linfeng Shi}
\DTLnewrow{names} \DTLnewdbentry{names}{name}{Tao Zhang}
\DTLnewrow{names} \DTLnewdbentry{names}{name}{Junbo Zhao}
\DTLnewrow{names} \DTLnewdbentry{names}{name}{Zhenzhong Lan}
\DTLnewrow{names} \DTLnewdbentry{names}{name}{Xudong Wang}
\DTLnewrow{names} \DTLnewdbentry{names}{name}{Jun Zhou$^{\dagger}$}
\DTLnewrow{names} \DTLnewdbentry{names}{name}{Zhiqiang Zhang$^{\dagger}$}
\DTLnewrow{names} \DTLnewdbentry{names}{name}{Meng Li}
\DTLnewrow{names} \DTLnewdbentry{names}{name}{Yalin Zhang}
\DTLnewrow{names} \DTLnewdbentry{names}{name}{Lei Chen}
\DTLnewrow{names} \DTLnewdbentry{names}{name}{Chao Huang}
\DTLnewrow{names} \DTLnewdbentry{names}{name}{Zhe Li}
\DTLnewrow{names} \DTLnewdbentry{names}{name}{Wang Hong}
\DTLnewrow{names} \DTLnewdbentry{names}{name}{Xiangchun Wang}
\DTLnewrow{names} \DTLnewdbentry{names}{name}{Congqi Li}
\DTLnewrow{names} \DTLnewdbentry{names}{name}{Yongzhen Guo}
\DTLnewrow{names} \DTLnewdbentry{names}{name}{Yuanyuan Wang}
\DTLnewrow{names} \DTLnewdbentry{names}{name}{Feng Xu}
\DTLnewrow{names} \DTLnewdbentry{names}{name}{Quanrui Guo}
\DTLnewrow{names} \DTLnewdbentry{names}{name}{Yilong Wang}
\DTLnewrow{names} \DTLnewdbentry{names}{name}{Xuemin Yang}
\DTLnewrow{names} \DTLnewdbentry{names}{name}{Chao Zhang}
\DTLnewrow{names} \DTLnewdbentry{names}{name}{Cheng Lin}
\DTLnewrow{names} \DTLnewdbentry{names}{name}{Lisha Liao}
\DTLnewrow{names} \DTLnewdbentry{names}{name}{Li Tang}
\DTLnewrow{names} \DTLnewdbentry{names}{name}{Baihui Li}
\DTLnewrow{names} \DTLnewdbentry{names}{name}{Wenbo Yu}
\DTLnewrow{names} \DTLnewdbentry{names}{name}{Dingbo Yuan}
\DTLnewrow{names} \DTLnewdbentry{names}{name}{Jianwen Wang}
\DTLnewrow{names} \DTLnewdbentry{names}{name}{Haonan Zheng}
\DTLnewrow{names} \DTLnewdbentry{names}{name}{Guojie Li}
\DTLnewrow{names} \DTLnewdbentry{names}{name}{Yingying Xu}
\DTLnewrow{names} \DTLnewdbentry{names}{name}{Mingchun Chen}
\DTLnewrow{names} \DTLnewdbentry{names}{name}{Yuefan Wang}
\DTLnewrow{names} \DTLnewdbentry{names}{name}{Siba Chen}
\DTLnewrow{names} \DTLnewdbentry{names}{name}{Tiwei Bie}
\DTLnewrow{names} \DTLnewdbentry{names}{name}{Zehuan Li}
\DTLnewrow{names} \DTLnewdbentry{names}{name}{Tianyu Zhou}
\DTLnewrow{names} \DTLnewdbentry{names}{name}{Anqi Shen}
\DTLnewrow{names} \DTLnewdbentry{names}{name}{Weihua Chen}
\DTLnewrow{names} \DTLnewdbentry{names}{name}{Tianyu Zhang}
\DTLnewrow{names} \DTLnewdbentry{names}{name}{Jianhao Fu}
\DTLnewrow{names} \DTLnewdbentry{names}{name}{Wengang Zheng}
\DTLnewrow{names} \DTLnewdbentry{names}{name}{Lianhao Xu}
\DTLnewrow{names} \DTLnewdbentry{names}{name}{Yicheng Shan}

\DTLsort{name}{names}

\large{Authors are listed \textbf{alphabetically by the first name}.} 

\large{
\begin{multicols}{3}
\raggedcolumns
Ling Team\\
\DTLforeach*{names}{\thename=name}{\thename\\}
\end{multicols}}

$^{\dagger}$ denotes corresponding authors.

\clearpage

\bibliographystyle{assets/plainnat}
\bibliography{main}

\clearpage
\beginappendix
\section{Related Work}
\subsection{Stable Reinforcement Learning} 
To accelerate the training process of reinforcement learning on large-scale language models, current RL frameworks employ different backend implementations for training and inference stages. However, such heterogeneity inevitably introduces unaligned calculations of token probability, which consequently brings instability issues to the training process. To solve this problem, GSPO~\citep{zheng2025gspo} previously employed a routing replay training strategy, which cached the activated experts in the old policy model with parameter $\ \theta_{\text{old}}$ in advance and replayed these routing modes in the current policy model when computing the importance ratios. However, the routing replay strategy will increase both computation and memory overhead to the RL framework. Rather, they proposed an algorithm-level solution, which defines a sequence-level importance ratio and performs clipping on that, avoiding training crashes caused by high-variance gradients. In comparison to GSPO, the techniques employed by IcePop are independent of sequence-level optimization, suggesting that they can be incorporated into ongoing research efforts. TIS~\citep{yao2025offpolicy} also addresses the probability discrepancy problem that exists between the training and inference stages. They proposed using importance sampling correction to address the divergence and demonstrated promising results. Different from our work, for gradients exhibiting significant divergence, \textit{i.e.}, tokens outside our masking range, TIS opts to keep updating those tokens by applying a moderating coefficient to their gradients. Empirically, we find that as training progresses, these minor disturbances can gradually amplify and ultimately lead to a plateau in benchmark performance.

\subsection{Efficient Reinforcement Learning}
For reasoning models that generate long sequences, the rollout phase often consumes a large percentage of training resources and time. To improve the training efficiency, TPPO~\citep{fan2025tppo} extends PPO by implementing a truncated rollout strategy and eliminating biased estimation from incomplete trajectories. Prior work~\citep{fu2025areal} incorporates a high-efficiency generation management suitable for an asynchronous RL system with interruptible rollout workers. 
In contrast to the existing works addressing low utilization of computing resources at the inference stage, C3PO++ stands out by applying a dynamic cutoff to rollout generation based on a token budget, balancing the stabilization of training updates with improved inference efficiency. Meanwhile, it supports high-throughput inference procedures in parallel, maximizing compute utilization and alleviating the rollout bottleneck during training. Our empirical analysis shows that the design of C3PO++ improves inference efficiency without sacrificing training performance, providing a solid foundation for integrating advanced training algorithms.

\subsection{Reinforcement Learning Infrastructure}
Training large-scale reinforcement learning models, particularly at the trillion-parameter level, imposes extraordinary demands on the underlying infrastructure. Below, we review how existing systems address the following three primary challenges and highlight their limitations, which collectively motivated the design of ASystem.

\begin{itemize}
\item \textbf{Memory Efficiency:} Managing the massive GPU memory footprint of model states, activations, and experience data throughout the training cycle without introducing significant overhead remains an open problem. Popular training and inference frameworks—including vLLM~\citep{kwon2023efficient}, SGLang~\citep{zheng2024sglangefficientexecutionstructured}, Megatron-LM~\cite{shoeybi2020megatronlmtrainingmultibillionparameter}, and RL-specific systems such as VeRL~\citep{sheng2024hybridflow} and OpenRLHF~\citep{hu2025openrlhfeasytousescalablehighperformance}—typically retain model states and communication groups in GPU memory throughout execution, leading to static and inefficient memory usage. The NVIDIA Collective Communication Library (NCCL) does not natively support live memory offloading. Although it provides plugin interfaces, a general and efficient solution remains absent. A recent effort, Slime~\citep{slime_github}, proposes destroying and re-creating NCCL communication groups to free memory, but the subsequent re-initialization overhead is prohibitive at scale, often consuming several minutes and severely disrupting training stability.

\item \textbf{State Synchronization:} Efficiently and reliably propagating model weights across distributed training and rollout workers is essential for policy consistency. Early RL frameworks such as OpenRLHF~\citep{hu2025openrlhfeasytousescalablehighperformance} relied on distributed file systems (e.g., NFS) for checkpoint sharing, suffering from limited bandwidth and throughput with synchronization latencies of tens of minutes. More recent systems, including VeRL~\citep{sheng2024hybridflow} and the checkpoint-engine~\citep{kimi_engine}, have shifted toward using NCCL for direct peer-to-peer weight synchronization. While this reduces dependency on shared storage, the approach still suffers from redundant data movement and poor performance in handling numerous small tensors, keeping end-to-end synchronization in the minute range.

\item \textbf{System Orchestration \& Reproducibility:} Providing a flexible, stable, and deterministic execution environment is crucial for rapid iteration. Existing frameworks like OpenRLHF and VERL are typically architected with tight coupling between components, making backend integration costly and slow. Furthermore, rollout behavior in these systems remains inherently non-deterministic due to fluctuating batch sizes and the non-associative nature of floating-point arithmetic across distributed workers~\citep{he2025nondeterminism}. The absence of systematic support for deterministic execution and metric alignment complicates reliable ablation studies and hinders root-cause analysis when reward improvement stagnates.

\end{itemize}

In contrast to the limitations above, ASystem is designed from the ground up to holistically address these core infrastructure challenges. Our \textbf{AMem} component enables efficient live memory offloading without the excessive re-initialization overhead seen in approaches like Slime~\citep{slime_github}. The \textbf{AState} system overcomes the inefficiencies of existing synchronization methods through a zero-redundancy P2P transmission mechanism and a hardware-aware multi-transport communication layer, supporting sub-second weight synchronization and in-place updates. Finally, the \textbf{Hybrid Runtime} offers a unified, RL-specific API that abstracts away backend complexities in networking, memory management, and weight exchange. Crucially, it incorporates a comprehensive precision alignment mechanism—comprising Tracker, Analyzer, and Replayer modules—to ensure end-to-end reproducibility and deterministic execution. Together, these components allow ASystem to deliver the stability, efficiency, and developer agility necessary for training massive models such as our \model{}.

\section{Theoretical Analysis for IcePop}\label{app:anlaysis_icepop}

\begin{theorem}[Compounding probability discrepancy]\label{thm:prob_dis_full}
Let $\pi_{\mathrm{infer}}(\cdot;\theta)$ and $\pi_{\mathrm{train}}(\cdot;\theta)$ be the policy model loaded by inference and training engines, and denote the probability discrepancy as
\[
\delta(\theta)\;=\;D_{\mathrm{KL}}\!\big(\pi_{\mathrm{infer}}(\cdot;\theta)\,\|\,\pi_{\mathrm{train}}(\cdot;\theta)\big), 
\qquad
\delta_t:=\delta(\theta_t).
\]
Consider the update with a step size of $\mu$ 
\[
\theta_{t+1}=\theta_t+\mu\,g_t,\qquad 
g_t \;=\; \mathbb{E}_{a\sim \pi_{\mathrm{infer}}(\cdot;\theta_t)}\!\big[A(a)\,\nabla_\theta\log\pi_{\mathrm{train}}(a;\theta_t)\big],
\]
and write $g_t=g_t^\star+b_t$ with
\[
g_t^\star \;=\; \mathbb{E}_{a\sim \pi_{\mathrm{train}}(\cdot;\theta_t)}\!\big[A(a)\,\nabla_\theta\log\pi_{\mathrm{train}}(a;\theta_t)\big],\qquad
b_t:=g_t-g_t^\star .
\]

Assume there exists a neighborhood that contains the iterates $\{\theta_t\}$ where

\smallskip
\noindent{(A1) $L$-smoothness.} $\delta$ is differentiable with $L$-Lipschitz gradient, i.e.
\[
\big|\,\delta(\theta+\Delta)-\delta(\theta)-\langle\nabla \delta(\theta),\Delta\rangle\,\big|
\;\le\; \tfrac{L}{2}\,\|\Delta\|^2.
\]

\noindent{(A2) Bias alignment.} There is $c>0$ such that
\[
\big\langle \nabla \delta(\theta_t),\, b_t \big\rangle \;\ge\; c\,\delta_t .
\]

\noindent{(A3) Bounded on-policy drift.} There is $M\ge 0$ such that
\[
\big|\big\langle \nabla \delta(\theta_t),\, g_t^\star \big\rangle\big| \;\le\; M .
\]

\noindent{(A4) Local gradient bound.} There is $G\ge 0$ such that $\|g_t\|\le G$.

\smallskip
Then for any stepsize $\mu\in(0,\bar\mu]$ (with $\bar\mu$ chosen so that (A1)–(A4) hold throughout the trajectory), there exist constants
\[
\eta:=c>0, 
\qquad 
\kappa:=M+\tfrac{L}{2}\,\bar\mu\,G^2\;\ge 0,
\]
such that
\[
\delta_{t+1}\;\ge\; \big(1+\eta\,\mu\big)\,\delta_t \;-\; \kappa\,\mu .
\]
Consequently, if $\delta_t \,\ge\, \delta_c:=\tfrac{2\kappa}{\eta}$, then
\[
\delta_{t+1}\;\ge\;\big(1+\tfrac{\eta}{2}\,\mu\big)\,\delta_t .
\]
\end{theorem}

\begin{proof}
By (A1) with $\Delta=\mu g_t$,
\[
\delta_{t+1}
= \delta(\theta_t+\mu g_t)
\;\ge\; \delta_t + \mu \big\langle \nabla \delta(\theta_t), g_t \big\rangle - \tfrac{L}{2}\mu^2\|g_t\|^2.
\]
Decompose $g_t=g_t^\star+b_t$ and apply (A2)–(A3):
\[
\big\langle \nabla \delta(\theta_t), g_t \big\rangle
= \big\langle \nabla \delta(\theta_t), g_t^\star \big\rangle
  + \big\langle \nabla \delta(\theta_t), b_t \big\rangle
\;\ge\; -M + c\,\delta_t .
\]
Use (A4) to bound the quadratic term:
\[
-\tfrac{L}{2}\mu^2\|g_t\|^2 \;\ge\; -\tfrac{L}{2}\mu^2 G^2 \;\ge\; -\tfrac{L}{2}\bar\mu\,\mu\,G^2 .
\]
Combine the inequalities to obtain
\[
\delta_{t+1}
\;\ge\; \delta_t + \mu\,(c\,\delta_t - M) - \tfrac{L}{2}\bar\mu\,\mu\,G^2
\;=\; (1+\eta\mu)\,\delta_t \;-\; \big(M+\tfrac{L}{2}\bar\mu G^2\big)\mu .
\]
Setting $\eta:=c$ and $\kappa:=M+\tfrac{L}{2}\bar\mu G^2$ gives the first claim:
\[
\delta_{t+1}\;\ge\;(1+\eta\mu)\,\delta_t - \kappa\,\mu .
\]

For the compounding form, rewrite
\[
(1+\eta\mu)\,\delta_t - \kappa\mu
\;=\; \Big(1+\tfrac{\eta}{2}\mu\Big)\delta_t \;+\; \Big(\tfrac{\eta}{2}\delta_t - \kappa\Big)\mu .
\]
Hence, if $\delta_t \ge 2\kappa/\eta$, the last term is nonnegative and
\[
\delta_{t+1}\;\ge\;\Big(1+\tfrac{\eta}{2}\mu\Big)\delta_t .
\]
\end{proof}

\section{Preliminary Analysis for IcePop on Ring-mini-2.0}
We also analyze IcePop in terms of the probability discrepancy between training and inference engines, training stability, exploration ability, and ill-conditioned tokens.
\paragraph{Training Stability.} We believe that a stable training process serves as a solid foundation and sufficient space to showcase the power of reinforcement learning. It is worth noting that both IcePop and TIS mitigate the instability of RL training within 600 gradient steps (see Figure \ref{fig:training_stability}), avoiding rapid training crashes occurring in the baseline setting. 

\begin{figure}[!hbt]
\centering
\begin{subfigure}[b]{0.45\textwidth}
\centering
\includegraphics[width=\textwidth]{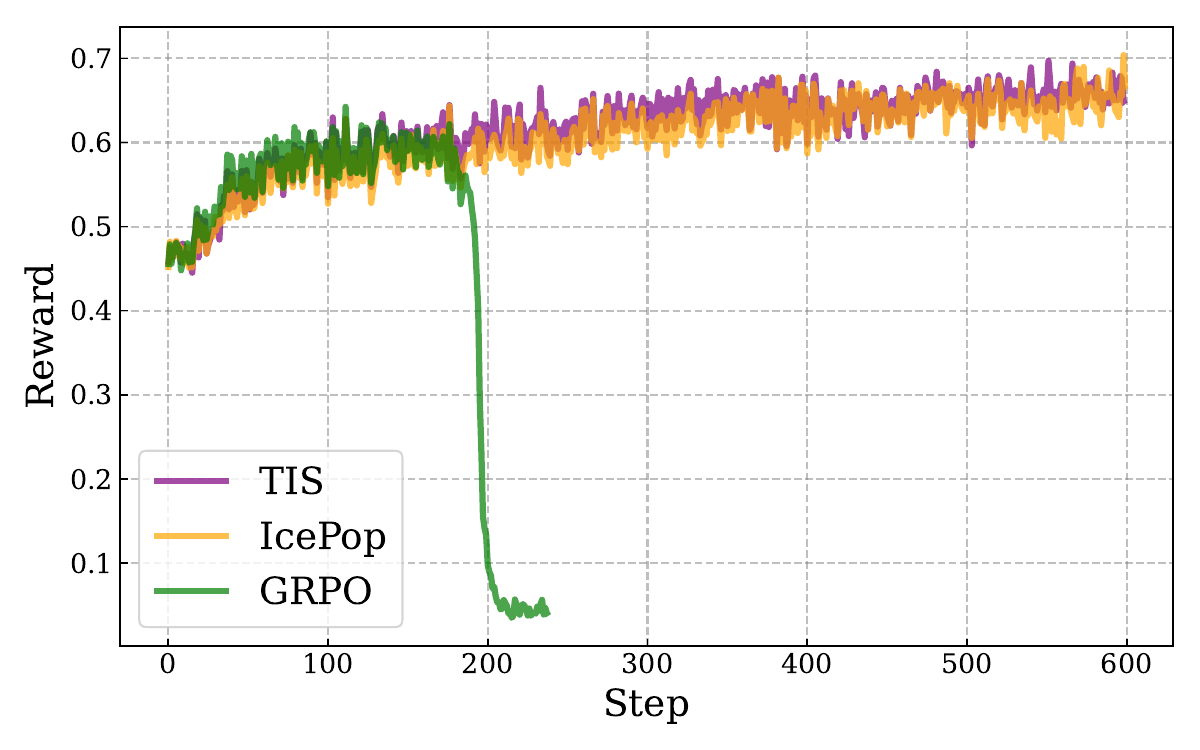}
\end{subfigure}
\hspace{2mm}
\begin{subfigure}[b]{0.45\textwidth}
\centering
\includegraphics[width=\textwidth]{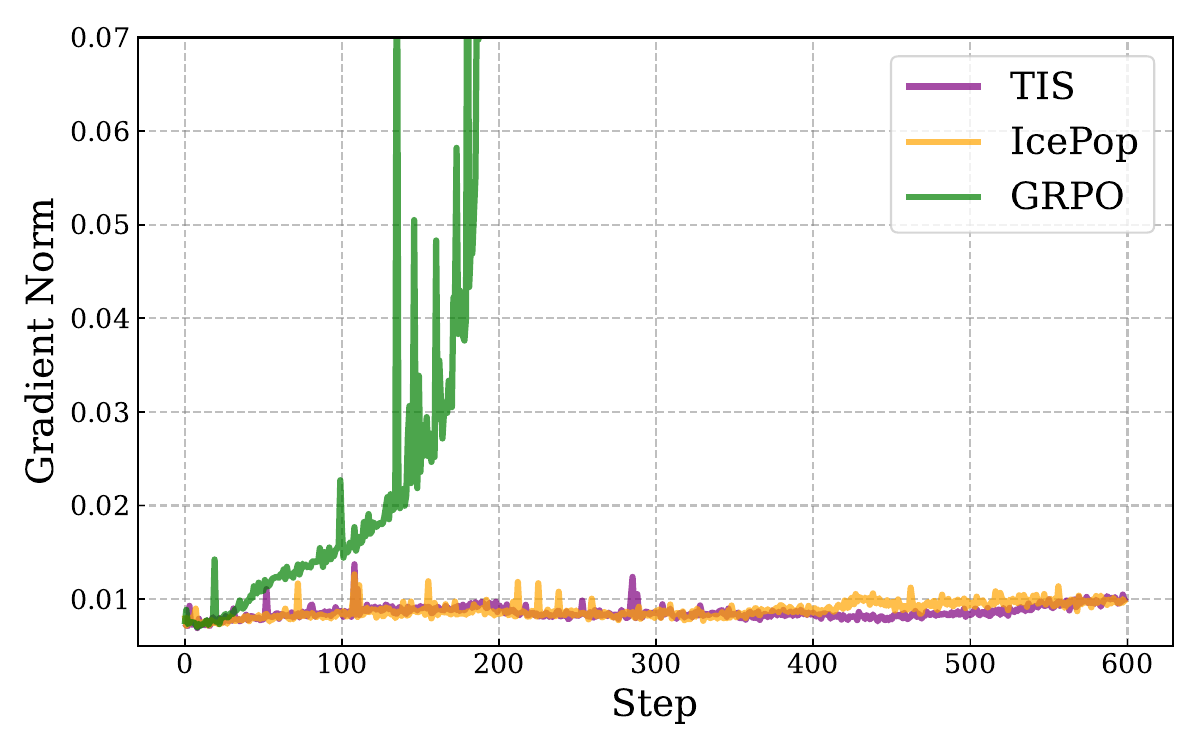}
\end{subfigure}
\caption{(Left) Training reward. The reward of baseline collapses after 180–200 steps. Both IcePop and TIS maintain stable growth. (Right) Gradient norm. Baseline explodes, IcePop and TIS remain stable.}
\label{fig:training_stability}
\end{figure}

\paragraph{Probability Discrepancy.} Without addressing the mismatch issues, the probability difference grows rapidly, as shown in the baseline setting. In contrast, both TIS and IcePop keep the KL divergence of training-inference probability within a reasonable range. Although the maximum probability difference rises for all three methods as training proceeds, the discrepancy of IcePop remains relatively low and even decreases within 400 steps (see the left in Figure \ref{fig:icepop-logp}). We also notice that TIS consistently shows larger extreme discrepancies and faster growth than ours, probably due to including the noisy policy updates during training.

\paragraph{Exploration Space.} In the right of Figure \ref{fig:icepop-logp}, we observed that the log probabilities of IcePop consistently maintain relatively lower values than those of TIS, which implicitly indicates that our method avoids overconfident predictions, thus ensuring a larger scope for exploring space, where low-probability tokens are more likely to be chosen, eventually increasing the diversity of responses.
    
\begin{figure}[!htb]
  \centering
  \begin{subfigure}[b]{0.45\textwidth}
    \centering
    \includegraphics[width=\textwidth]{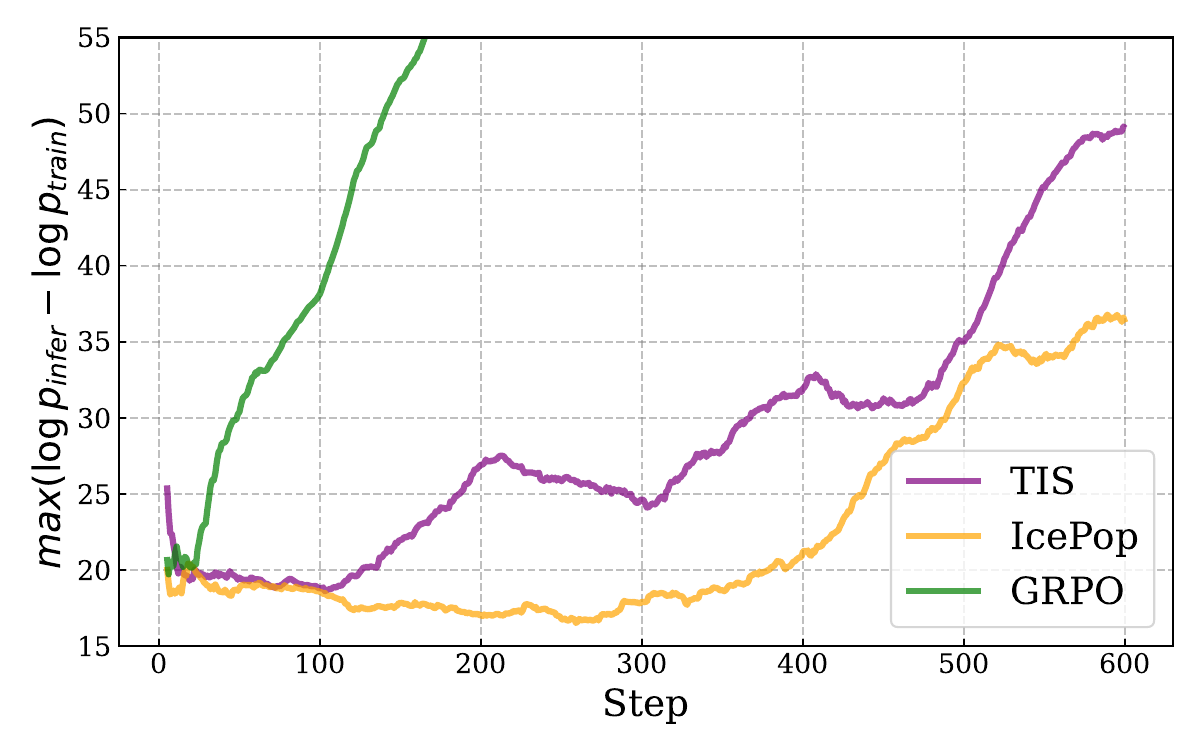}
  \end{subfigure}
  \hspace{2mm}
  \begin{subfigure}[b]{0.45\textwidth}
    \centering
\includegraphics[width=\textwidth]{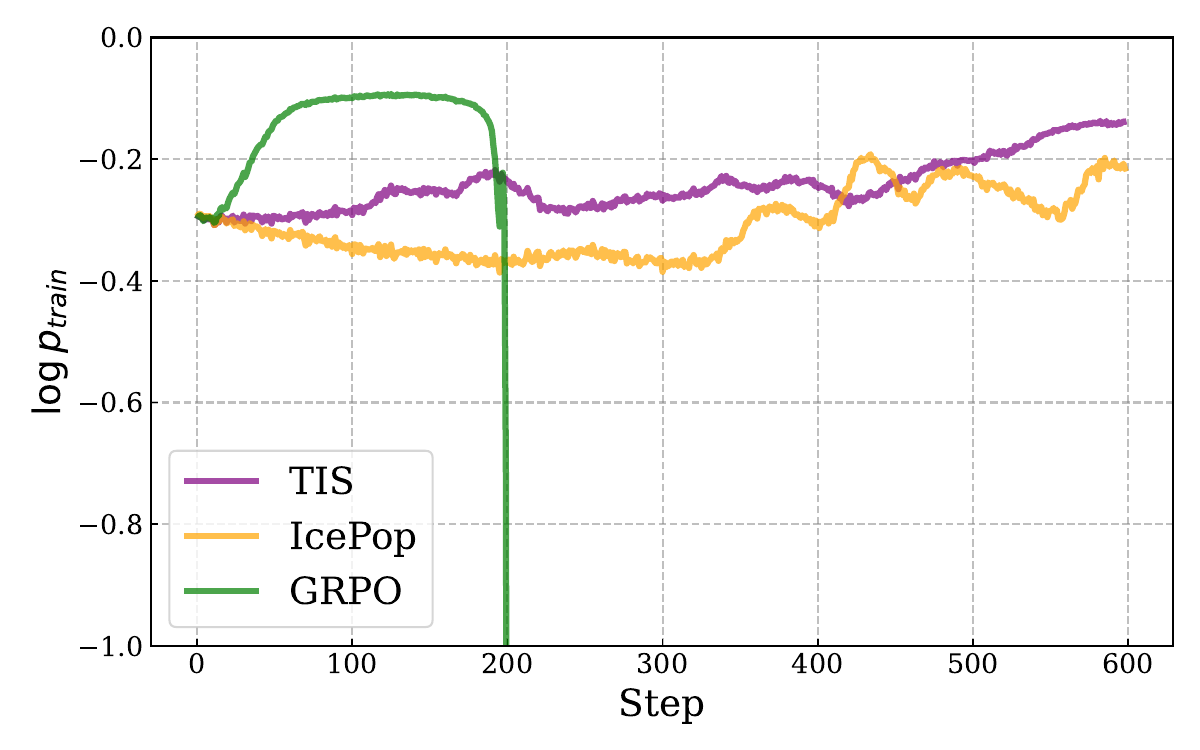}
  \end{subfigure}
  \caption{(Left) The maximum of probability discrepancy. (Right) The log probability of tokens. Baseline increases rapidly and drops to the bottom, while IcePop is relatively steady.}
\label{fig:icepop-logp}
\end{figure}
    
\paragraph{Ill-conditioned Tokens.} In our experiments, we found that the clipping ratio from our masking mechanism stays around 1–2‰ of training tokens (see Figure \ref{fig:ill_conditioned_tokens}). As training progresses, the clipping ratio rises sharply, suggesting that increasingly subtle but harmful gradient updates occur and necessitate a higher clipping ratio. We also conducted a detailed analysis of the clipped tokens. The following right figure shows that, compared to all tokens, clipped tokens have higher entropy, indicating that the clipped tokens play an important role in training.

\begin{figure}[!hbt]
\centering
\begin{subfigure}[b]{0.45\textwidth}
\centering
\includegraphics[width=\textwidth]{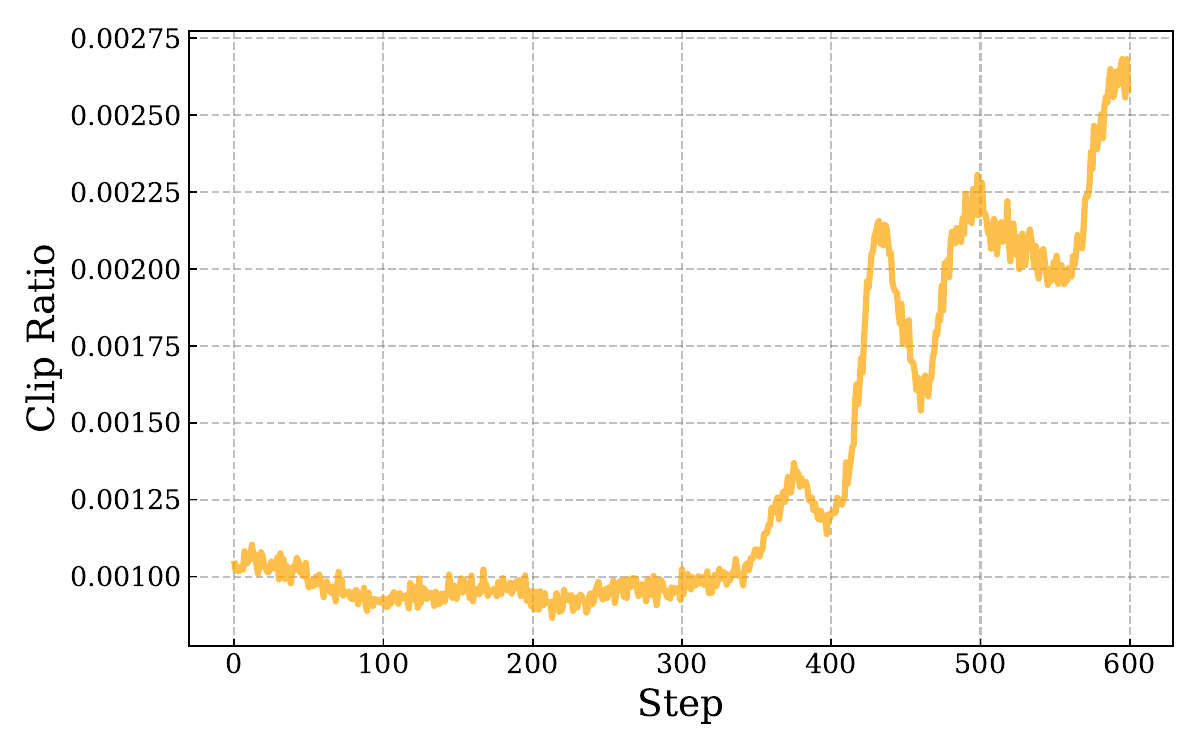}
\end{subfigure}
\hspace{2mm}
\begin{subfigure}[b]{0.49\textwidth}
\centering
\includegraphics[width=\textwidth]{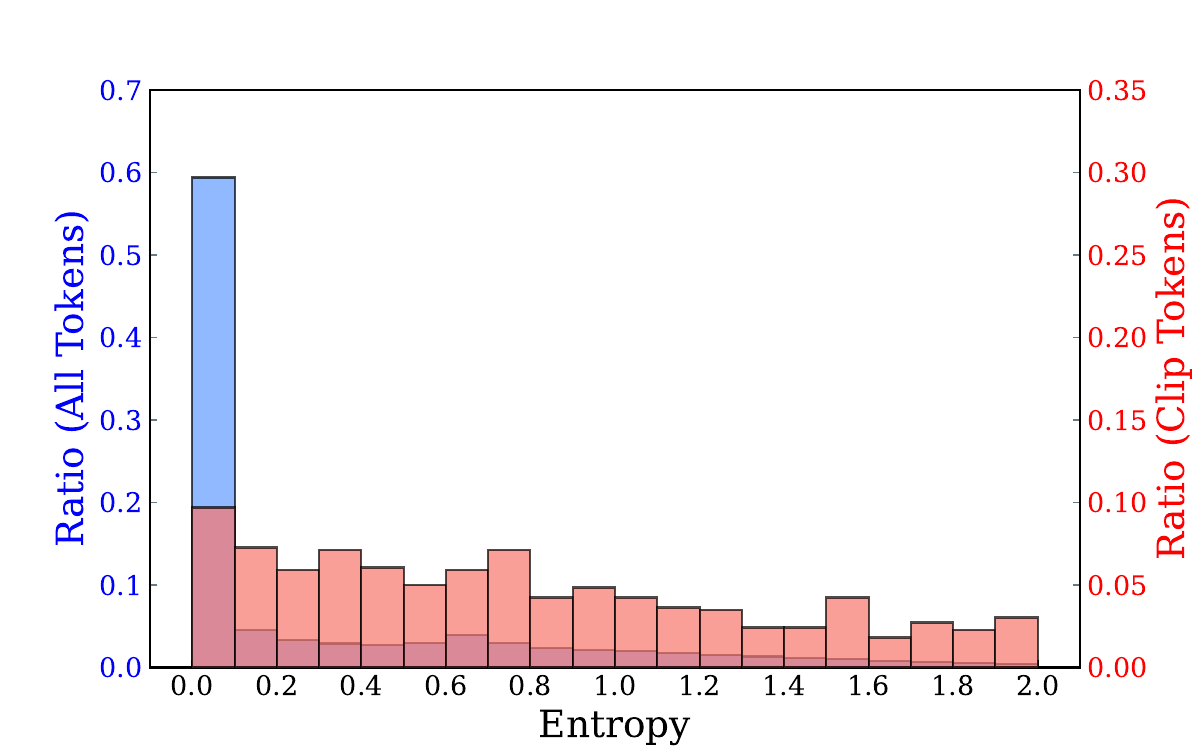}
\end{subfigure}
\caption{(Left) Clipping ratio. IcePop maintains ~1–2‰  of tokens clipped in our default setting. (Right) The comparisons of token entropy between all tokens and clipped tokens. Compared to all tokens, clipped tokens show a higher proportion of high-entropy tokens.}
\label{fig:ill_conditioned_tokens}
\end{figure}

\paragraph{Sensitivity Analysis.} We compare how different masking ranges for the calibration ratio affect training. Specifically, we experiment with three masking ranges: (1) $\alpha=0.5, \beta=5.0$ (default range), (2) $\alpha=0.5, \beta=2.0$ (narrow range), and (3) $\alpha=0.4, \beta=5.0$ (wider range). As shown in Figure \ref{fig:ablation_study},
\begin{enumerate}
    \item The default masking range $[0.5, 5.0]$ not only stabilizes training but also enriches sampling diversity.
    \item The narrow masking range $[0.5, 2.0]$ immediately destabilizes training, as shown in the volatility of gradient norm and the sharp increase in probability discrepancy.
    \item The wide masking range $[0.4, 5.0]$ still stabilizes training, yet includes tokens with higher log probability compared to the default setting.
\end{enumerate}

\begin{figure}[!hbt]
  \centering
  \begin{subfigure}[b]{0.45\textwidth}
    \centering
    \includegraphics[width=\textwidth]{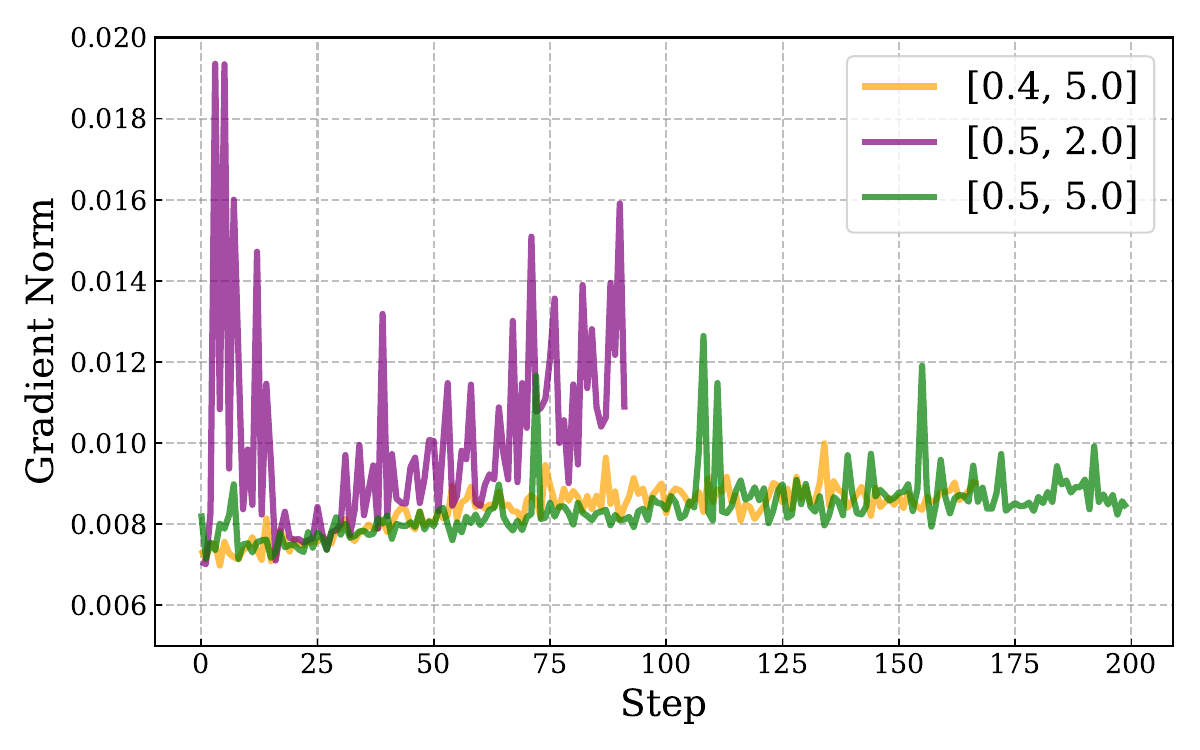}
  \end{subfigure}
  \hspace{2mm}
  \begin{subfigure}[b]{0.45\textwidth}
    \centering
\includegraphics[width=\textwidth]{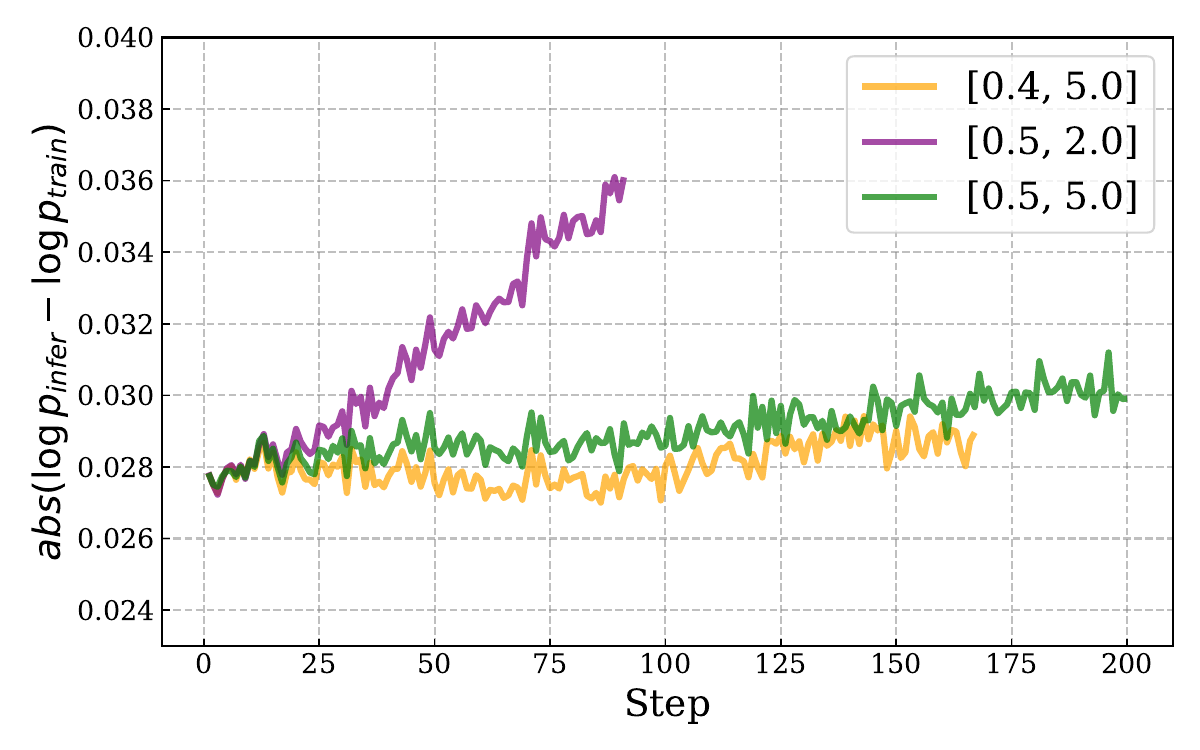}
  \end{subfigure}
    \begin{subfigure}[b]{0.45\textwidth}
    \centering
    \includegraphics[width=\textwidth]{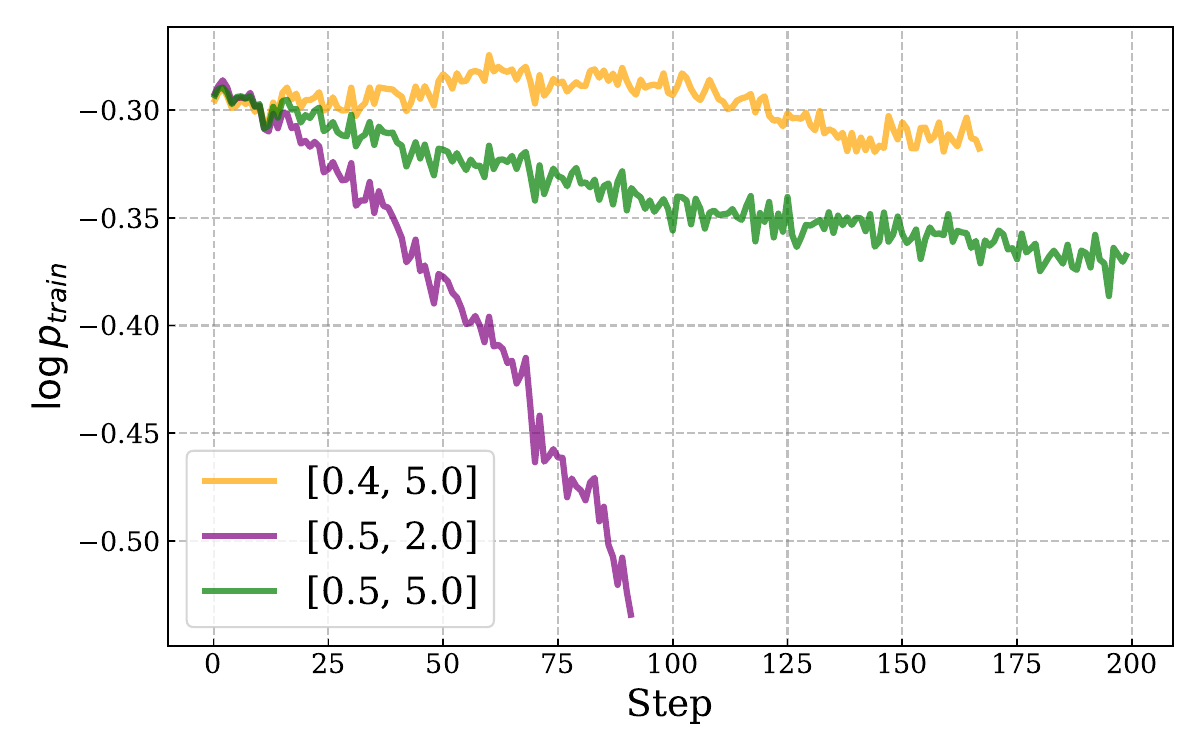}
  \end{subfigure}
  \hspace{2mm}
  \begin{subfigure}[b]{0.45\textwidth}
    \centering
\includegraphics[width=\textwidth]{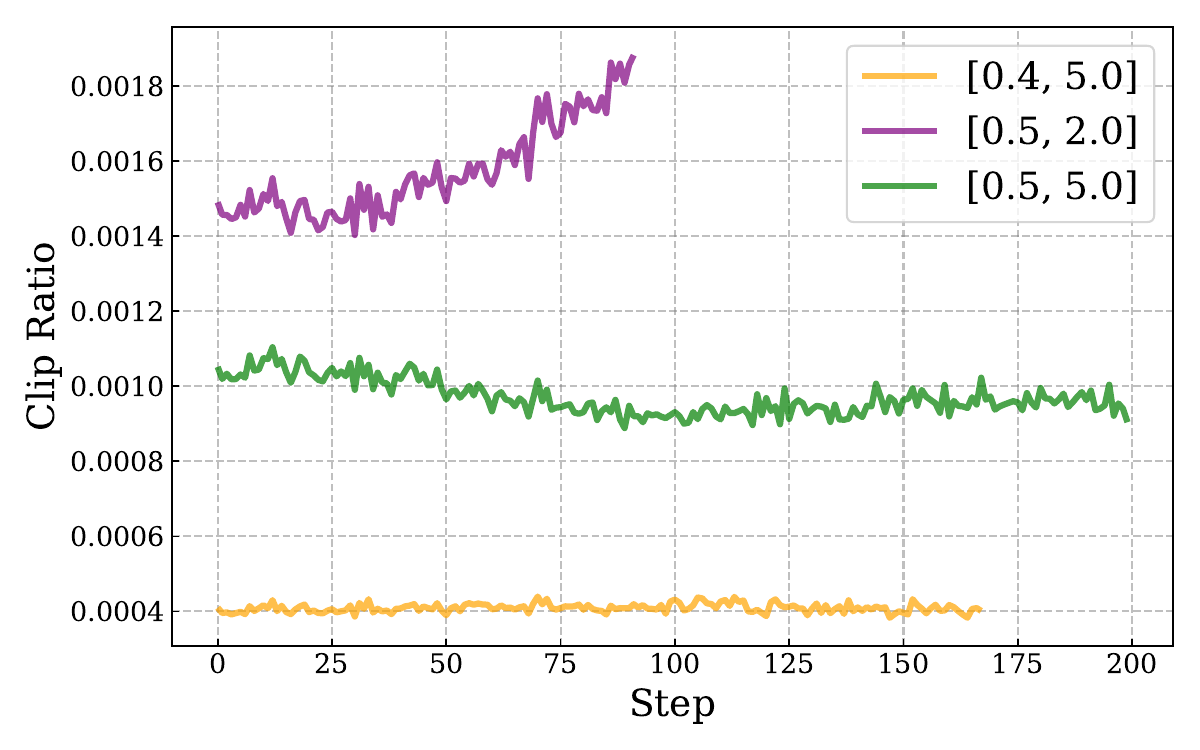}
  \end{subfigure}
  \caption{The training dynamics under different masking ranges.}
\label{fig:ablation_study}
\end{figure}

\section{Training Data Analysis}

This section analyzes the composition of our SFT and RL datasets. We first present the domain distribution of the SFT data in Figure~\ref{fig:sft_data_dist}, illustrating its diversity, which underpins the model's broad knowledge base. Subsequently, Figure~\ref{fig:rl_data_dist} details the complexity distribution of the RL data.

\begin{figure}[!htb]
    \centering
    \includegraphics[width=0.8\linewidth]{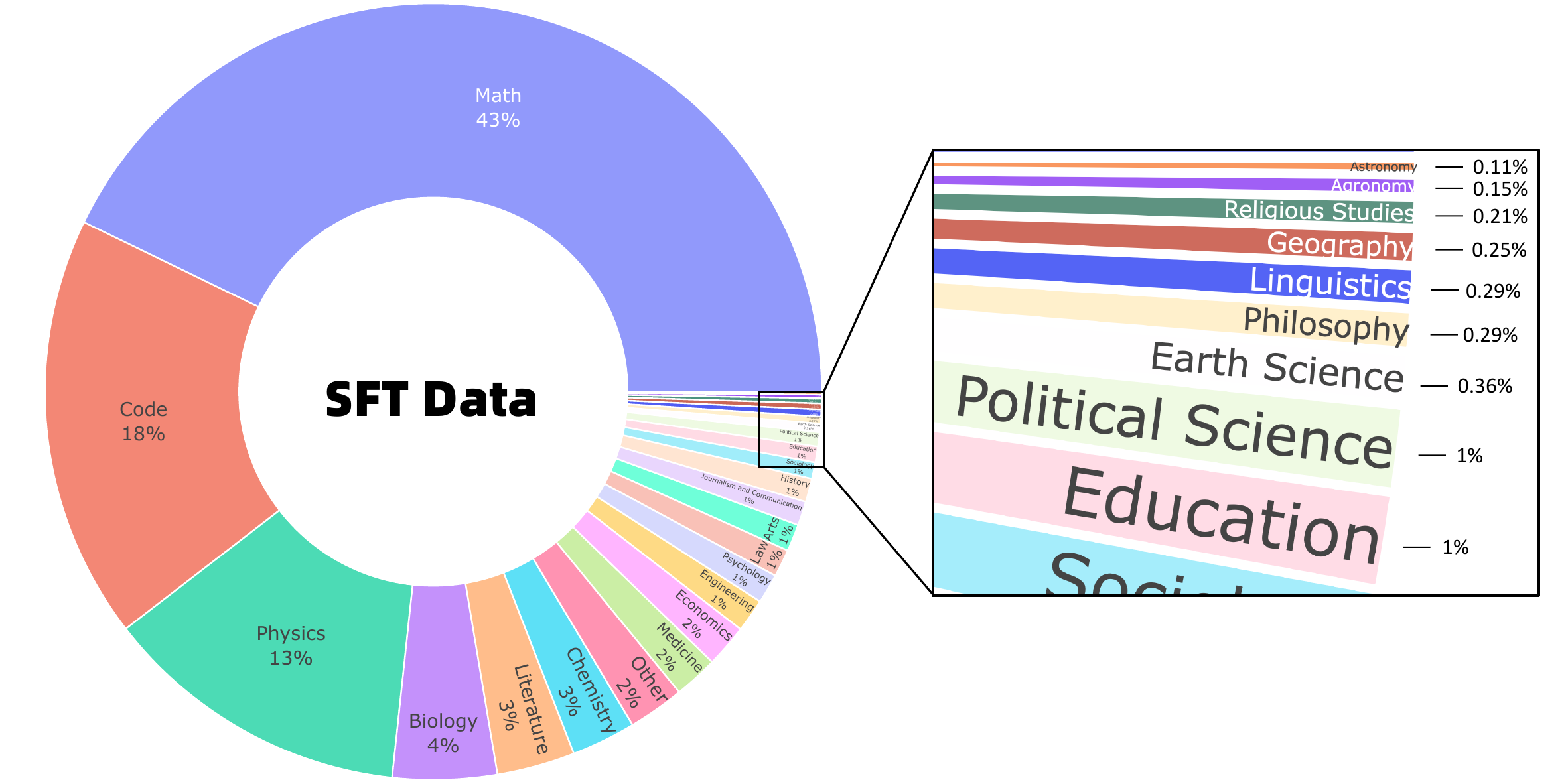}
    \caption{The domain distribution of SFT Data.}
    \label{fig:sft_data_dist}
\end{figure}

\begin{figure}[!htb]
    \centering
    \includegraphics[width=1\linewidth]{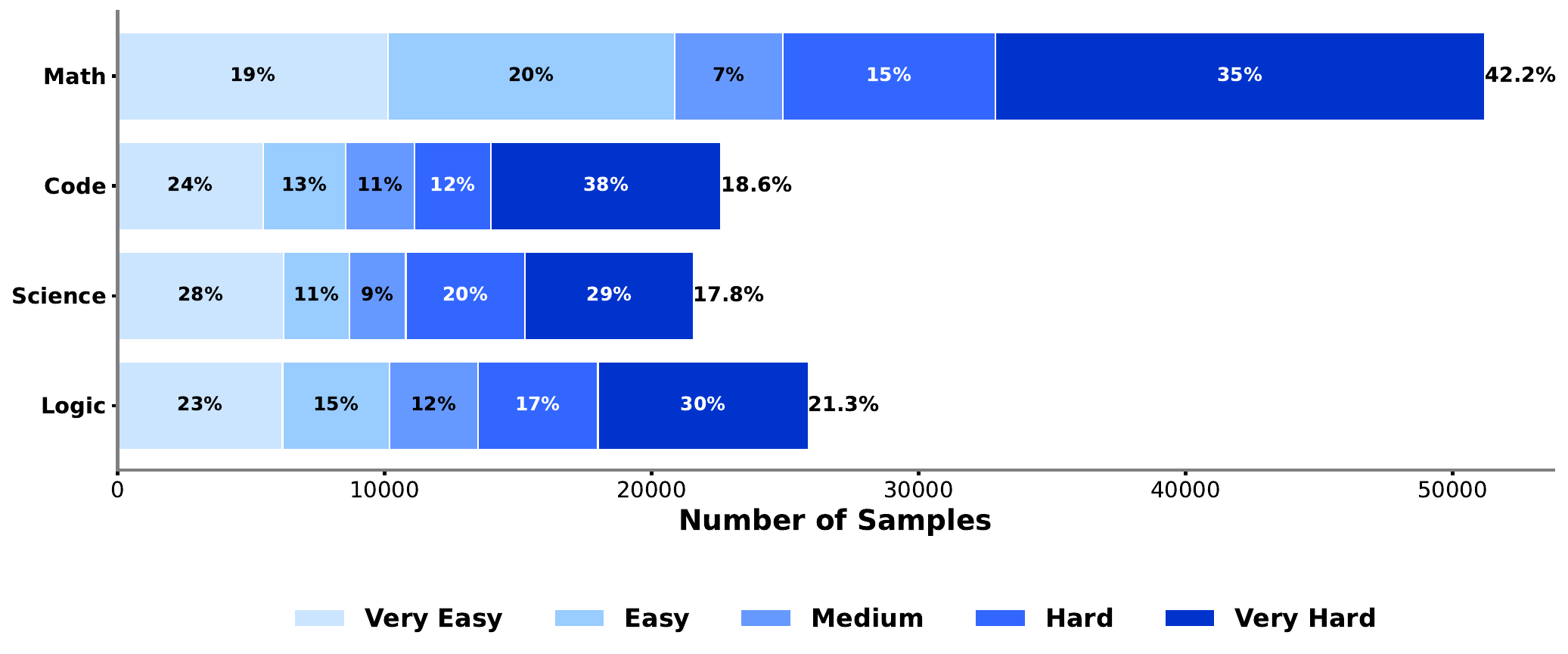}
    \caption{The difficulty distribution of RL Data.}
    \label{fig:rl_data_dist}
\end{figure}

\section{IMO Case}
\label{sec:imo}
As established in the main paper, our proposed model \model{} achieved a silver-medal performance at IMO 2025. The detail IMO 2025 Results from \model{} could refer to the AWorld~\footnote{\url{https://github.com/inclusionAI/AWorld/tree/main/examples/imo/samples/samples\%20from\%20Ring-1T}}. Below, we present a detailed solution from \model{} to the competition's first problem:

\paragraph{Question:} A line in the plane is called sunny if it is not parallel to any of the $x$–axis, the $y$–axis, or the line $x + y = 0$. Let $n \geq 3$ be a given integer. Determine all nonnegative integers $k$ such that there exist n distinct lines in the plane satisfying both of the following: 1. for all positive integers $a$ and $b$ with $a + b \leq n + 1$, the point $(a, b)$ lies on at least one of the lines; and 2. exactly $k$ of the $n$ lines are sunny.

\paragraph{Solution:} Generated by \model{}: 

\begin{tcolorbox}[
    width=0.85\linewidth,
    center,
    breakable,
    colback=white,
    colframe=black,
    boxrule=1pt,
    arc=4pt,
    fontupper=\scriptsize,
]
\section*{Summary}

\subsection*{a. Verdict}
I have found a complete solution. The nonnegative integers $k$ satisfying the conditions are precisely $k = 0$, $k = 1$, and $k = 3$. All other values of $k$ (i.e., even integers $k \geq 2$, odd integers $k \geq 5$, and $k > n$) are impossible.

\subsection*{b. Method Sketch}
The problem requires covering all points $(a, b) \in \mathbb{Z}^+ \times \mathbb{Z}^+$ with $a + b \leq n + 1$ using exactly $n$ distinct lines, where exactly $k$ are ``sunny'' (not parallel to the $x$-axis, $y$-axis, or $x + y = 0$).

\begin{itemize}
    \item \textbf{Key Observations:}
    \begin{itemize}
        \item The set of required points forms a triangular grid $P_n$ with $T(n) = \frac{n(n+1)}{2}$ points.
        \item Non-sunny lines fall into three categories: horizontal ($y = c$), vertical ($x = d$), and diagonal ($x + y = e$).
        \item A sunny line can cover at most $\left\lfloor \frac{k+1}{2} \right\rfloor$ points from any subset isomorphic to $P_k$, due to geometric constraints on collinearity in the grid.
    \end{itemize}
    
    \item \textbf{Critical Results:}
    \begin{itemize}
        \item \textbf{Even $k \geq 2$:} The minimum number of uncovered points $U(k) = \frac{k(k+1)}{2}$ exceeds the maximum coverage $k \cdot \frac{k}{2}$ by $k$ sunny lines, making it impossible.
        \item \textbf{Odd $k \geq 5$:} For $k = 2t - 1$ with $t \geq 3$, the minimal uncovered set $P_k$ has $t(2t - 1)$ points. While $k$ sunny lines could theoretically cover this if each had $t$ points, only 3 disjoint $t$-point sunny lines exist in $P_k$, leaving $2t(t - 2)$ points that cannot be covered by the remaining $2t - 4$ lines (as $2t(t - 2) > (2t - 4)(t - 1)$ for $t \geq 3$).
        \item \textbf{$k = 0, 1, 3$:} Explicit constructions show these are always possible:
        \begin{itemize}
            \item $k = 0$: Use $n$ diagonal lines $x + y = 2, \dots, n + 1$.
            \item $k = 1$: Use $n - 1$ horizontal lines and one sunny line through the single uncovered point.
            \item $k = 3$: Use $n - 3$ horizontal lines to leave a subgrid isomorphic to $P_3$, which can be partitioned into 3 sunny lines.
        \end{itemize}
    \end{itemize}
    
    \item \textbf{Construction Strategy:}
    \begin{itemize}
        \item For $k = 3$, the uncovered subgrid $P_3$ is covered by three specific sunny lines with slopes $1$, $-2$, and $-\frac{1}{2}$, verified explicitly for $P_3$ and generalized via coordinate transformation for larger $n$.
    \end{itemize}
\end{itemize}

\section*{Detailed Solution}

\subsection*{Step 1: Understanding the Point Set}
Let $N = n + 1$. The set of required points is:
\[
P_n = \{(a, b) \in \mathbb{Z}^+ \times \mathbb{Z}^+ : a + b \leq N\}.
\]
This set has $T(n) = \frac{n(n+1)}{2}$ points. Each point lies on one horizontal line ($y = b$), one vertical line ($x = a$), and one diagonal line ($x + y = s$).

\subsection*{Step 2: Maximum Coverage by Non-Sunny Lines}
Define $M(m, n)$ as the maximum number of points in $P_n$ covered by $m$ non-sunny lines. It is known that:
\[
M(m, n) = \frac{m(2n - m + 1)}{2},
\]
achieved by selecting the largest $m$ non-sunny lines (e.g., top $m$ rows). If $m = n - k$ non-sunny lines are used, the minimum number of uncovered points is:
\[
U(k) = T(n) - M(n - k, n) = \frac{k(k+1)}{2}.
\]
These $U(k)$ points must be covered by $k$ sunny lines.

\subsection*{Step 3: Impossibility for Even $k \geq 2$}
For even $k \geq 2$, $U(k) = \frac{k(k+1)}{2}$. Any sunny line can cover at most $\left\lfloor \frac{k+1}{2} \right\rfloor = \frac{k}{2}$ points from a set isomorphic to $P_k$. Thus, $k$ sunny lines cover at most:
\[
k \cdot \frac{k}{2} = \frac{k^2}{2} < \frac{k(k+1)}{2} = U(k).
\]
Hence, \textbf{no configuration exists} for even $k \geq 2$.

\subsection*{Step 4: Impossibility for Odd $k \geq 5$}
Let $k = 2t - 1$ with $t \geq 3$. Then $U(k) = t(2t - 1)$. The maximum number of points per sunny line in $P_k$ is $t$, but only 3 disjoint $t$-point sunny lines exist in $P_k$ (e.g., slopes $1$, $-2$, $-\frac{1}{2}$), covering $3t$ points. The remaining points are:
\[
t(2t - 1) - 3t = 2t(t - 2).
\]
With $k - 3 = 2t - 4$ lines left, each can cover at most $t - 1$ points, yielding maximum coverage:
\[
(2t - 4)(t - 1) = 2t^2 - 6t + 4.
\]
Since $2t(t - 2) = 2t^2 - 4t > 2t^2 - 6t + 4$ for $t \geq 3$, the remaining points cannot be covered. Thus, \textbf{no configuration exists} for odd $k \geq 5$.

\subsection*{Step 5: Possibility for $k = 0$}
Use the $n$ lines $x + y = 2, 3, \dots, n + 1$. These are all non-sunny (slope $-1$) and partition $P_n$, covering all $T(n)$ points. Hence, $k = 0$ is valid.

\subsection*{Step 6: Possibility for $k = 1$}
Use $n - 1$ horizontal lines $y = 1, 2, \dots, n - 1$, covering all points with $b \leq n - 1$. The only uncovered point is $(1, n)$ (since $a + n \leq n + 1 \implies a = 1$). Cover $(1, n)$ with a sunny line, e.g., $y = x + (n - 1)$ (slope $1$, sunny). This line is distinct from the horizontal lines, so all $n$ lines are distinct and exactly $k = 1$ is sunny.

\subsection*{Step 7: Possibility for $k = 3$}
Use $m = n - 3$ horizontal lines $y = 1, 2, \dots, m$, covering all points with $b \leq m$. The uncovered points satisfy $b \geq m + 1$, $a \geq 1$, $a + b \leq N$. Substituting $b' = b - m$, this becomes $a + b' \leq 4$, forming a subgrid isomorphic to $P_3$ with 6 points:
\[
(1,1), (1,2), (1,3), (2,1), (2,2), (3,1).
\]
Cover $P_3$ with three sunny lines:
\begin{itemize}
    \item $L_1$: $(1,1), (2,2)$ (slope $1$),
    \item $L_2$: $(1,2), (3,1)$ (slope $-\frac{1}{2}$),
    \item $L_3$: $(1,3), (2,1)$ (slope $-2$).
\end{itemize}
These lines are sunny (slopes not $0, \infty, -1$) and disjoint, covering all 6 points. Transforming back to original coordinates ($b = b' + m$), the lines remain distinct and sunny. Combined with the $m$ non-sunny horizontal lines, this gives $n$ distinct lines with exactly $k = 3$ sunny lines.

\subsection*{Step 8: Completeness of Characterization}
\begin{itemize}
    \item $k = 0$: Always possible via diagonal lines.
    \item $k = 1$: Always possible via $n - 1$ horizontal lines and one sunny line.
    \item $k = 3$: Always possible for $n \geq 3$ via the above construction.
    \item Even $k \geq 2$: Impossible due to coverage deficit.
    \item Odd $k \geq 5$: Impossible due to geometric constraints on $P_k$.
    \item $k > n$: Impossible since there are only $n$ lines.
\end{itemize}
No other values of $k$ are feasible.

\section*{Final Answer}
All nonnegative integers $k$ satisfying the conditions are $k = 0$, $k = 1$, and $k = 3$. Thus, the solution is:
\[
\boxed{k = 0,\ 1,\ 3}
\]
\end{tcolorbox}

\end{document}